\documentclass[letterpaper]{article} 
\usepackage{aaai25}  
\usepackage{times}  
\usepackage{helvet}  
\usepackage{courier}  
\usepackage[hyphens]{url}  
\usepackage{graphicx} 
\usepackage{natbib}  
\usepackage{caption} 
\usepackage{cite}
\usepackage{algorithm}
\usepackage{algorithmic}
\usepackage{amsmath}
\usepackage{appendix}
\usepackage{color}
\usepackage{caption}

\usepackage{newfloat}
\usepackage{listings}
\DeclareCaptionStyle{ruled}{labelfont=normalfont,labelsep=colon,strut=off} 
\floatstyle{ruled}
\newfloat{listing}{tb}{lst}{}
\floatname{listing}{Listing}
\title{Know2Vec: A Black-Box Proxy for Neural Network Retrieval}
\author{
    Zhuoyi Shang\textsuperscript{\rm 1,2,3}, Yanwei Liu\textsuperscript{\rm 1,3}\thanks{Corresponding authors.}, Jinxia Liu\textsuperscript{\rm 4}, Xiaoyan Gu\textsuperscript{\rm 1,3}, Ying Ding\textsuperscript{\rm 1 3}, Xiangyang Ji\textsuperscript{\rm 5}
}
\affiliations{
    \textsuperscript{\rm 1}Institute of Information Engineering, Chinese Academy of Sciences, Beijing, China   \\
     \textsuperscript{\rm 2}School of Cyber Security, University of Chinese Academy of Sciences, Beijing, China\\
     \textsuperscript{\rm 3}Key Laboratory of Cyberspace Security Defense, Beijing, China \\
     \textsuperscript{\rm 4}College of Information and Intelligence Engineering, Zhejiang Wanli University, Ningbo, China \\
     \textsuperscript{\rm 5}Tsinghua University, Beijing, China

\textbraceleft shangzhuoyi,liuyanwei\textbraceright @iie.ac.cn, liujinxia@zjwu.edu.cn,  \textbraceleft guxiaoyan,dingying\textbraceright @iie.ac.cn, xyji@tsinghua.edu.cn

}

\usepackage{bibentry}

\begin{document}
	
	\maketitle
	
	\begin{abstract}
		
		For general users, training a neural network from scratch is usually challenging and labor-intensive. Fortunately, neural network zoos enable them to find a well-performing model for directly use or fine-tuning it in their local environments. Although current model retrieval solutions attempt to convert neural network models into vectors to avoid complex multiple inference processes required for model selection, it is still difficult to choose a suitable model due to inaccurate vectorization and biased correlation alignment between the query dataset and models. From the perspective of knowledge consistency, i.e., whether the knowledge possessed by the model can meet the needs of query tasks,  we propose a model retrieval scheme, named Know2Vec, that acts as a black-box retrieval proxy for model zoo. Know2Vec first accesses to models via a black-box interface in advance, capturing vital decision knowledge from models while ensuring their privacy. Next, it employs an effective encoding technique to transform the knowledge into precise model vectors.  Secondly, it maps the user's query task to a knowledge vector by probing the semantic relationships within query samples. 
		Furthermore, the proxy ensures the knowledge-consistency between query vector and model vectors within their alignment space, which is optimized through the supervised learning with diverse loss functions, and finally it can identify the most suitable model for a given task during the inference stage. Extensive experiments show that our Know2Vec achieves superior retrieval accuracy against the state-of-the-art methods in diverse neural network retrieval tasks.

	\end{abstract}
	\begin{links}
			\link{Code}{https://github.com/vimpire00/know2vec1}
	\end{links}

	\section{Introduction}
	Well-trained models in many domains have demonstrated promising performance in various downstream tasks. The training process refines knowledge from dataset into general rules and patterns, enabling the model to make accurate predictions on new data \cite{KR}. However, their performances vary widely for a targeted downstream application \cite{Spider}. The model whose knowledge is more closely aligned with the task requirements tends to perform better.
	For example, a model with numerical knowledge would find it easier to complete the MNIST\cite{mnist} classification task than a model specialized in flower classification.
	Assessing the suitableness of a Deep learning(DL) model by uploading the entire dataset to the huge model market for comparison against inference results is risky and impractical due to data disclosure and resource constraints. Therefore,  further research is needed to evaluate the correlation between neural network models and query tasks.

	
	Source-Free model transferability estimation (SF-MTE) \cite{HScore,LEEP,Spider} methods are designed to rank the suitability of pre-trained models for fine-tuning in downstream tasks. Traditional methods \cite{HScore,LEEP} directly score the candidate models by utilizing statistical data like features or joint distribution of models and query task. Typically, Model Spider\cite{Spider} vectorizes both neural network models and query tasks  to avoid the high computational costs of forward propagation increased by traditional methods.

	\begin{figure}[t]
		\centering
		\includegraphics[width=0.97\columnwidth]{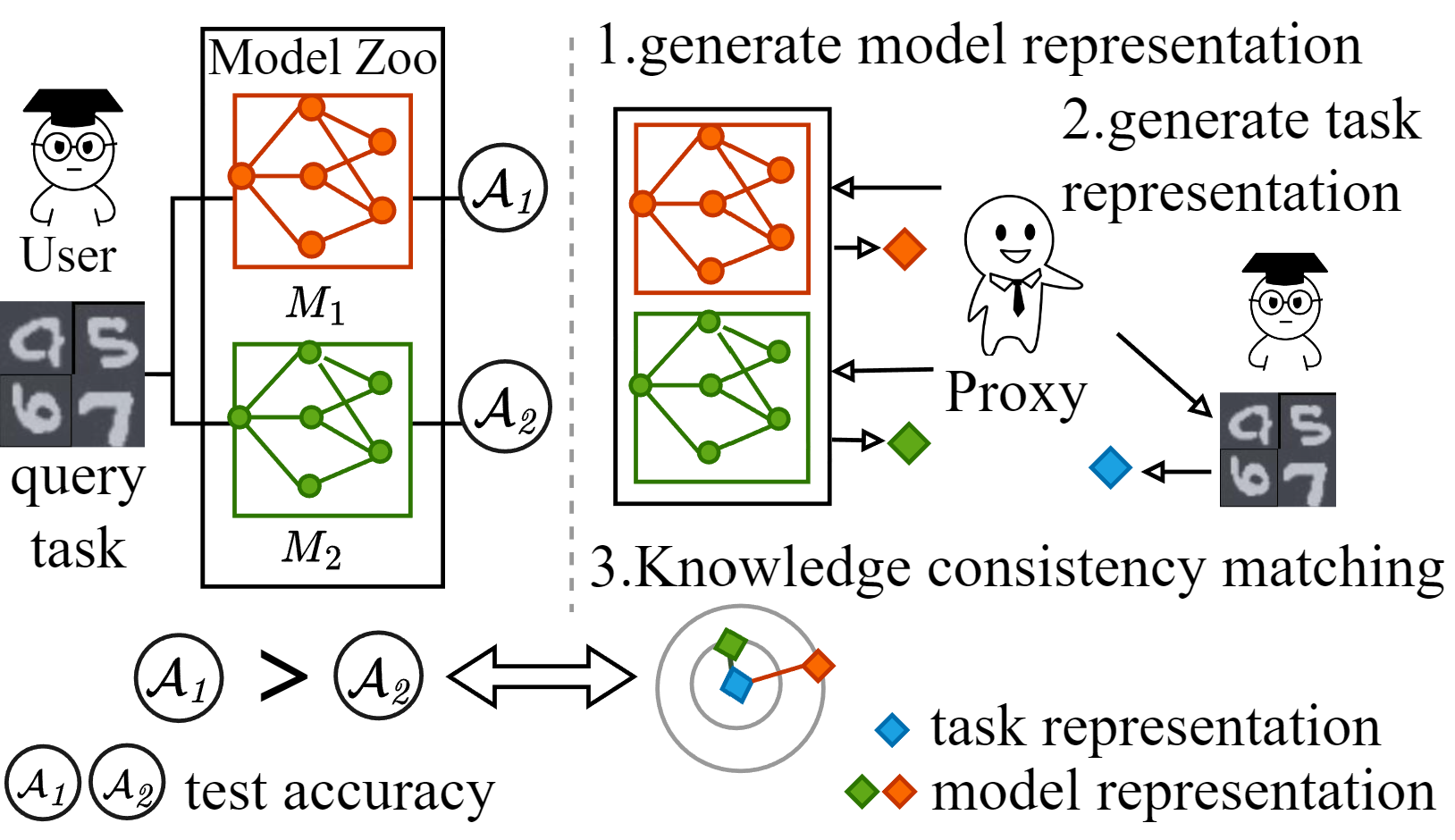} 
		\caption{Knowledge-consistency-based
			black-box proxy for model retrieval.}
		\label{fig1}
	\end{figure}
	
	With the vectorization idea, Neural Network Retrieval (NNR) \cite{TANS,DNNR} tries to transform models and datasets into specific embeddings that facilitates their matching.
	Generating vectors for models and datasets, and calculating their correlations require an accurate understanding of the key knowledge  of both models and query tasks. The pioneering NNR study, DNNR \cite{DNNR} utilizes litmus images to construct models' semantic vectors, while TANS \cite{TANS} further advances the field by searching for a cross-modal space to minimize the semantic discrepancy between model representations and query images. These techniques, while improving retrieval efficiency, still encounter various problems,  such as laborious and rough vector generation process \cite{DNNR,Spider}, imprecise alignment \cite{TANS}, and  the necessity for  privacy protection \cite{Spider,TANS}.  
	
	In particular, the primary challenge of the correlation calculation methods(SF-MTE or NNR) lies in the two aspects as follows. (1) \textbf{Transforming the unstructured nature of neural network models into a vectorial format}, which must capture the intrinsic knowledge in models for effective retrieval.    (2)\textbf{ The establishment of a quantifiable mapping space}, where query vectors align with model representations, ensuring semantically similar vectors are proximate.  Existing methods often rely on complex and suboptimal vector generation processes, failing to fully capture critical model knowledge or achieve seamless alignment.
	
	To address the above challenges, we propose a novel knowledge-consistency-based black-box proxy for model retrieval, named Know2Vec, and it is shown in Fig. \ref{fig1}.
	The objective of Know2Vec is to establish a consistent representation of knowledge, allowing semantic alignment between  the query task and models. Firstly, it abstracts the intrinsic knowledge acquired by neural network models into a generalized representation in a black-box way. Next, it interacts with users to generate effective task representation by understanding the differences between query samples. Lastly, the proxy is designed to perform a knowledge consistency matching between the abstracted model representation and the task representation, facilitating efficient model retrieval.  
	
	Our  key contributions are:
	\begin{itemize}
		\item 
		We propose a model knowledge vectorization scheme for parameter-agnostic scenarios, which is designed to capture the implicit model knowledge and further vectorize it to support accurate model retrieval. We further prove in theory that it is feasible to obtain model information with randomly selected probes. 
		\item  A carefully designed measure function is proposed to align the heterogeneous knowledge embeddings,  which correspond to the knowledge of the query task and those of known models, assisting users in accurately defining their needs and retrieving the most suitable neural network model. 
		\item 
		Know2Vec achieves superior retrieval performance across various NNR tasks, outperforming state-of-the-art baselines in our experiments. Additionly, 
		it accesses neural network models in a black-box manner, eliminating the need to understand internal parameters, thus preserving privacy.

	\end{itemize}
	
	\section{Related Work}
	
	\subsection{Neural Network Retrieval}
	NNR addresses the model selection issue by mapping query entries and neural network models into vectors, enabling users to find a satisfactory pre-trained model from model markets \cite{Learnware}.
	Deep Neural Network Retrieval (DNNR)\cite{DNNR} initially achieves model vectorization through feeding random litmus images to the candidate models. However, it  needs extensive datasets and computational resources, making it impractical for online retrieval. TANS \cite{TANS} aims to align query datasets with similar neural network representations, but it overlooks the subtle differences within categories that are key for aligning knowledge.
	By representing key decision-making knowledge from models without accessing to their internal parameters and aligning it with the need of a query task, our method achieves both privacy protection and precise retrieval goals.
	
	
	
	\subsection{Source-Free Model Transferability Estimation}
	For a given target task and a model library, Source-Free Model Transferability Estimation(SF-MTE) \cite{whichtransfer} aims to propose a metric to quantify the transferability score without the need for individual training. Static SF-MTE methods, such as LEEP\cite{LEEP},
	H-score\cite{HScore},  compute scores directly from statistical data like features and logits. In contrast,  Dynamic SF-MTE methods aim to project static features into tailored spaces to facilitate superior approximation. They try to estimate the maximum average log evidence\cite{LogME}, or they endeavor to identify a model/task vectorlization technique such as the classical Model Spider \cite{Spider}.
	However, despite enhancing computational efficiency to a certain extent, these methods still necessitate a complex training process.

	\subsection{Boundary Supporting Samples}
	Boundary supporting samples are identified as those close to the decision boundary of neural network models.
	Assume the target model is a k-class DNN classifier, where the output layer is an active layer. 
	Formally, we denote by $ \{g_i\}$ the decision functions of the target classifier, 
	and a data point $x$ is on the target classifier’s classification boundary if at least two labels have the largest discrimination probability, i.e., $ g_a(x) = g_b(x) \geq \max\limits_{c \neq a, b} g_c(x)$, where $a,b,c$ are category index, and $g_a(x)$ is the probability that sample $x$ belongs to category $A$ \cite{IPGuard}. 
	Tian et al. \cite{KR} claimed that the knowledge transferred from a training dataset to a DL model
	can be uniquely represented by the model’s decision boundary samples, providing feasibility for us to acquire model knowledge in a black-box setting. However, this method only acquires partial model knowledge and  requires target training dataset, which is illogical in NNR problem. Accordingly, we propose a parameter-agnostic model knowledge vectoring approach without demanding training dataset.
	
	\begin{figure*}[t]
		\centering
		\includegraphics[width=1.8\columnwidth]{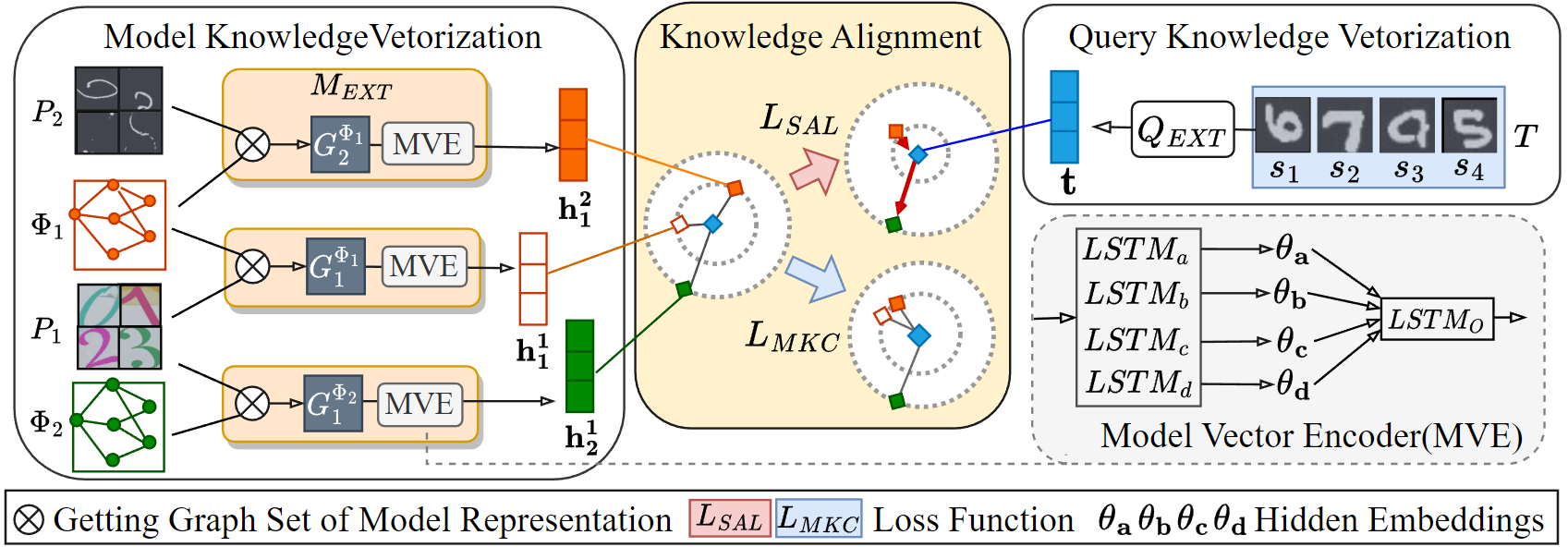} 
		\caption{Model Retrieval Framework. 
		}
		\label{framwork}
		\vspace{-0.4cm}
	\end{figure*}
	
	\section{Method}
	
	\label{method}
	Taking NNR problem as an example, we will elaborate on calculating model-dataset correlation when the model parameters are agnostic. We start by vectorizing models and query tasks, which helps to distill the models' knowledge and clarify the requirements of the tasks. We tackle the new issues that incurred from the limited known information.  
	Next, we seek a knowledge-consistent space that acts as a bridge, which connects the two modalities despite their differences in  structures and semantic parameters, and providing a way to measure their semantic similarity.
	
	\subsection{Problem Formulation}
	\label{formulation}
	We consider an arbitrary query task $T=\{s_i,l_i\}_{i=1}^n$ with $n$ samples $\{s_i\}$ and the corresponding target labels $\{l_i\}$. Given a large model hub $M=\{\Phi_i\}_{i=1}^m$ with a total of $m$ well-trained models, the goal of NNR is to choose 
	a DNN model $\Phi_j$ that performs well on $T$.  We define
	$\mathcal{A}(\Phi_i,T)$ the verification accuracy of $T$ on model $\Phi_i$. Mathematically, NNR aims to search for the best-fitted model $\Phi_j$ that satisfies
	\begin{equation}
		j=\mathop{\arg\max}\limits_{i}  \mathcal{A}(\Phi_i,T)
	\end{equation}
	
	As mentioned earlier, we assume there is a virtually perfect proxy $\mathcal{P}$ that serves as a good communication intermediary: 
	(1) It distills model knowledge and obtains vector $\mathbf{h_i}$ for each candidate model $\Phi_i$; (2) It gets the requirements of query tasks $T$ and generates the corresponding query knowledge vector $\mathbf{t}$; (3) It selects a suitable model $\Phi_j$  through a  semantic measurement $\mathcal{DIS}()$. For an effective NNR method, maximizing  $\mathcal{A}(\Phi_i,T)$  is equivalent to  minimizing $\mathcal{DIS}(\mathbf{h_i},\mathbf{t})$
	\begin{equation}
		\label{eq1} 
		j=\mathop{\arg\max}\limits_{i}  \mathcal{A}(\Phi_i,T)\Leftrightarrow j=\mathop{\arg\min}\limits_{i}  \mathcal{DIS}(\mathbf{h_i},\mathbf{t})
	\end{equation}
	
	Fig.~\ref{framwork} illustrates the  well-designed model retrieval framework. $\mathcal{P}$ includes three components: a model knowledge extractor $M_{EXT}$, a query knowledge extractor $Q_{EXT}$, and a knowledge alignment space with measurement function $\mathcal{DIS}()$.
	Firstly, $M_{EXT}$ vectorizes model knowledge assisted by a series of additional probe datasets, denoted as  $\mathbf{h_i^{i'}}=M_{EXT}(P_{i'},\Phi_i)$, where $P_{i'}$ is the $i'$th probe dataset. Specifically for model $\Phi_i$, $M_{EXT}$ starts by creating a graph set $G_{i'}^{\Phi_i}$ with $P_{i'}$, and then encodes $G_{i'}^{\Phi_i}$  into $\mathbf{h^{i'}_i}$ through a model vector encoder. Next, $Q_{EXT}$ extracts semantic correlations from query task $T$, producing a task knowledge vector $\mathbf{t}=Q_{EXT}(T)$.
	After that, the knowledge alignment space assesses the consistency of knowledge between model vectors $\{\mathbf{h}_i\}$ and query vector $\mathbf{t}$, for selecting the suitable model to the task.
	
	\subsection{Model Knowledge Vectorization}
	\label{modelkn}
	The previous NNR methods attempted to break down candidate model and explore the semantic information through exposed parameters. However, this is laborious and privacy-unfriendly. Fortunately, Theorem \ref{theorem11}  proves that model knowledge that is transferred from the training dataset can be encapsulated by a matrix, denoted as knowledge representation matrix ($KRM$), providing the possibility for more efficient and privacy-preserving model knowledge embedding. Given the neural network $\Phi$ and its centroid samples of training dataset $P$, $KRM$ can be generated based on the model's response to input samples in advance, which only requiring black-box access to $\Phi$.   
	
	$KRM$ is formed by two kind of representative samples: centroid samples and decision boundary samples.  
	As illustrated in Fig.~\ref{fig3}, taking binary classification that contains categories $A$ and $B$ as an example, the transferred knowledge $KRM$ consists of two vectors $\{ \mathbf{r_a^b}=x_a^b-x_a,\mathbf{r_b^a}=x_b^a-x_b  \}$, where $x_a$ and $x_b$ are centroid samples of $A$ and $B$, respectively, $x_a^b$ is the decision boundary sample from $A$ to $B$, and $x_b^a$ is the decision boundary sample from $B$ to $A$. $x_a^b$ can be generated by points that respectively belong to categories $A$ and $B$(such as $x_a,x_b$), and similarly for $x_b^a$. 
	
	Yet even with $KRM$, obtaining an effective representation $\mathbf{h}$ is still challenging. The primary obstacle lies in the design of model knowledge extractor $M_{EXT}$, which is responsible for converting information-limited $KRM$ into measurable vectors to enable model retrieval tasks. These vectors in $KRM$ encapsulates the incomplete decision knowledge of the model.
	Considering the dynamic changes in features, we further collect the information both in $KRM$ and representative samples as a graph set and design a specialized DL framework to generate  $\mathbf{h}$.
	
	Furthermore, obtaining  central samples is almost impossible because model owners tend to withhold their training datasets due to privacy concerns or copyright restrictions, which complicating the $KRM$ generation process. We solve this issue by proving that alternative datasets can effectively generate a model's knowledge representation vector.
	
	
	
	
	
	\newtheorem{theorem}{Theorem} 
	\begin{theorem}
		\label{theorem11}  
		\cite{KR}
		The knowledge transferred from a training dataset to a deep learning model can be represented by the knowledge representation matrix $KRM$ formed by perturbation vectors across different classes. For a k-class classifier, let the centroid sample of category $A$ be denoted as $x_a$, the perturbation vector $\mathbf{r_a^k}=x_a^k-x_a$ from category $A$ to category $K$ is defined as the offsets between  $x_a$ and $x_a^k$, where $x_a^k$ is the boundary sample from category $A$ to category $K$.   Collectively,
		$KRM$ is defined by:
		\begin{align}
			\small
			KRM=
			\begin{bmatrix}
				\mathbf{0} & \mathbf{r_{a}^b} & ... &  \mathbf{r_{a}^k}\\
				\mathbf{r_{b}^a} & \mathbf{0}&... & \mathbf{r_{b}^k}\\
				.. & ...  & ..  & \mathbf{r_{c}^k} \\
				\mathbf{r_{k}^a}  & \mathbf{r_{k}^b}  &...   & \mathbf{0} \\  
			\end{bmatrix}
			\label{eqkrm}
		\end{align}    
	\end{theorem}

	\begin{figure}[t]
		\centering
		\includegraphics[width=0.6\columnwidth]{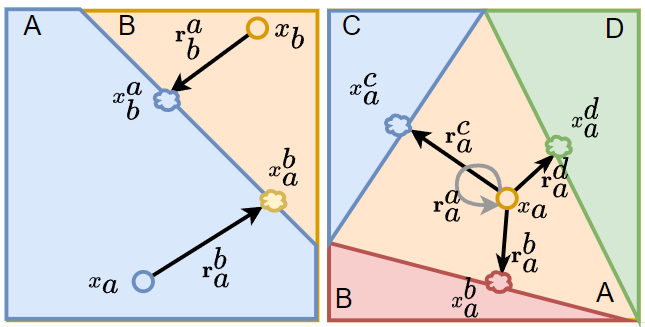} 
		\vspace{-0.25cm}
		\caption{Decision Boundary Sample.}
		\label{fig3}
		\vspace{-0.4cm}
	\end{figure}
	
	\subsubsection{Getting Graph Set of Model Representation}  
	Relying solely on $KRM$ to obtain model knowledge may lead to information loss since 
	the knowledge transfer vector $\mathbf{r_b^a}=x_b^a-x_b$ overlooks crucial details such as the starting point $x_b$ and ending point $x_{b}^a$. The central sample holds key feature about its category, while the boundary samples imply  transition features between categories. 
	First, it's difficult to pinpoint exactly how features changed as the sample moves from  $x_b$ to $x_{b}^a$ within  $\mathbf{r_b^a}$. Second, $KRM$  offers limited insight of the  distinctive  intra-class knowledge.
	
	As illustrated in Fig.~\ref{fig3}, for a classification model $\Phi_i$ with  4 categories, the centroid sample $x_a$ of category $A$   is interconnected with boundary samples $\overline{x_a} = \{ x_a^b, x_a^c,x_a^d\}$, forming a directed graph structure. Within this structure, the directed edges $\{ \mathbf{r_{a}^b},\mathbf{r_{a}^c},\mathbf{r_{a}^d}\}$ represent the specific connections of $x_a$.  This graph is formally defined as  $G_a=\{x_a,\overline{x_a},\{ \mathbf{r_{a}^b},\mathbf{r_{a}^c},\mathbf{r_{a}^d}\} \}$.  Among them,  $\overline{x_a}$ and $x_a$ are two different types of nodes.  Expand to other categories, a total of 4 sets of such connection relationships can be modeled: $G^{\Phi_i}=\{ G_a,G_b,G_c,G_d \}$. Undoubtedly, $G^{\Phi_i}$ offers a richer semantic representation than $KRM$.  
	%
	
	$G^{\Phi_i}$ implicitly links $G_a,G_b,G_c,G_d$ through relationships between categories. 
	Specifically, the information of category $A$  can be obtained from  these three types of nodes: (1) The central sample $x_a$ which embodies the unique features about $A$; (2) The boundary samples $\overline{x_a}$ from  $A$ to other categories, and they explain which features need to change for the transition from $A$ to other categories; (3) Boundary samples $x_b^a, x_c^a, x_d^a $ from other categories to $A$ that suggest why the model might incorrectly classify as $A$. These points are present in $G_b, G_c,G_d$.   By encoding all nodes in $G^{\Phi_i}$ through  inter-category relationships, we facilitate a transformation from  $G^{\Phi_i}$ into the model knowledge vector $\mathbf{h}$.
	
	\subsubsection{Implementation of Model Vector Encoder} 
	We consider the relationships as the dependencies of sequential data, in which each sequence corresponds to one category.  As shown in Fig.~\ref{framwork}, the model vector encoder is implemented with an inner-outer encoder.  The inner encoder processes individual subgraphs, while the outer encoder integrates these subgraphs. Both are completed by a bidirectional Long Short Term Memory(LSTM)\cite{LSTM} network to handle variable long term dependencies. For category $A$, the information from the above-mentioned first two types of nodes (1) and (2) has been successfully encoded to
	the hidden embeddings $\theta_a$ by inner layer $LSTM_a$, as mentioned in Eq.~\ref{eq10}. Additionally, the embeddings $\theta_b,\theta_c,\theta_d$ already include information from the third type of nodes (3). These embeddings are further aggregated as the model knowledge vector $\mathbf{h}$ by the outer-layer $LSTM_O$, aligned through sequence correspondence. 
	\begin{equation} 
		\begin{split} 
			\mathbf{\theta_a} &= LSTM_a(x_a,x_a^b,x_a^c,x_a^d; W, b) \\
			\mathbf{\theta_b} &= LSTM_b(x_b^a,x_b,x_b^c,x_b^d; W, b) \\
			\mathbf{\theta_c} &= LSTM_c(x_c^a,x_c^b,x_c,x_c^d; W, b) \\
			\mathbf{\theta_d} &= LSTM_d(x_d^a,x_d^b,x_d^c,x_d; W, b) \\
			\mathbf{h_i} &= LSTM_O(\mathbf{\theta_a},\mathbf{\theta_b},\mathbf{\theta_c},\mathbf{\theta_d}; W, b) 
			\label{eq10}
		\end{split}
	\end{equation} 
	where $W,b$ are the optimizable parameters.
	
	Thus,  $M_{EXT}$ has encoded $G^\Phi$ into vector $\mathbf{h}$,  as the bidirectional network ensures all edges  in  $G^\Phi$ are reachable, either directly or indirectly.
	
	\subsubsection{Using probe datasets instead of training datasets} Access to the training dataset of a neural network is sometimes impractical, but we can still use other data to probe and obtain boundary samples. Lemma \ref{lemma1} theoretically proves the feasibility that  we can still get the semantic relationships of $G^{\Phi_i}$ with probe samples.
	
	\newtheorem{lemma}{Lemma}
	\begin{lemma}
		\label{lemma1}
		The perturbation vectors in $KRM$ can also be obtained from the target model with associating the external  datasets.
	\end{lemma}
	\newtheorem{proof}{Proof}
	\begin{proof}
		Taking binary classification with categories $A$ and $B$ as an example,
		we consider a neural network model $\Phi$ as
		\begin{equation}
			\small
			\Phi(x)=\delta(\mathbf{w}*x+\mathbf{b})
		\end{equation}
		where $\delta$ is the active function, $\mathbf{w}$ and $\mathbf{b}$ are the weights and biases, respectively,  $x$ is any input sample, and $*$ denotes multiplication between vectors. 
		
		$\delta$ is composed of $g_A$ and $g_B$, where $g_A(x)$ is the probability that sample $x$ belongs to category $A$, and similarly $g_B(x)$ for category $B$. For the convenience of narration, it may be helpful to set $\delta=g_A-g_B$.
		Assuming that the centroid samples of training dataset for $\Phi$ are $x_a$ and $x_b$, the boundary sample from $A$ to $B$ is $x_a^b$ ,
		then 
		\begin{align}
			\small
			\delta(\mathbf{w}*x_a^b+\mathbf{b})=0\\
			\delta(\mathbf{w}*x_a+\mathbf{b})=1\\
			\delta(\mathbf{w}*x_b+\mathbf{b})=-1
		\end{align}
		There must exist two selected samples $z_a$ and $z_b$ that satisfy $\delta(\mathbf{w}*z_a+b)=1-\lambda_1,\delta(\mathbf{w}*z_b+b)=-1+\lambda_2$, where  $\lambda_1,\lambda_2$ are very small values that can be ignored. Correspondingly, a boundary sample $z_a^b$ from $A$ to $B$ 
		satisfies $\delta(\mathbf{w}*z_a^b+b)=0-\lambda_3$, with $\lambda_3$  
		being a very small value.
		Then, 
		\begin{align}
			\small
			\delta(\mathbf{w}*x_a^b+\mathbf{b})-\delta(\mathbf{w}*z_a^b+\mathbf{b})=\lambda_3\\
			\delta(\mathbf{w}*x_a+\mathbf{b})-\delta(\mathbf{w}*z_a+\mathbf{b})=\lambda_1 \\
			\delta(\mathbf{w}*x_b+\mathbf{b})-\delta(\mathbf{w}*z_b+\mathbf{b})=\lambda_2 
		\end{align}
		Since $\delta$ is continuous and differential in the regions of interest, there exists a value $\sigma$ that satisfies $z_a^b=x_a^b+\sigma$ due to the Mean Value Theorem. Similarly, there must also be a disturbance $\sigma_a$ such that $x_a=z_a+\sigma_a$.
		
		Therefore, the  perturbation vector $\mathbf{r_a^b}=x_a^b-x_a$ can also alternatively represented by Eq.(\ref{eqr}), 
		where $\sigma$ and $\sigma_a$ are the offsets.
		\begin{equation}
			\small
			\label{eqr}
			\mathbf{r_a^b}=z_a^b-z_a+\sigma+\sigma_a
		\end{equation}
		
		For a fixed model $\Phi$,  $x_a$ and $x_a^b$ are unique, and also $\sigma$ and $\sigma_a$ are only related to $z_a$ and $z_a^b$, respectively. Therefore,
		$\mathbf{r_a^b}$ can be represented by $z_a$ and $z_a^b$. This principle is applicable to other vectors in $KRM$, and the proven conclusion can be extended to other classification models.
		
	\end{proof}
	
	
	\subsection{Query Knowledge Vectorization}
	For the query task $T=\{s_i,l_i\}_{i=1}^n$, we implement the query encoder $Q_{EXT}$  to discern correlations both within and across categories within the query samples. We first average the samples of each class to identify features unique to that class, denoted as $\mathbf{\theta_k}$ for class $k$, then we feed these features of distinct classes as separate sequences into a bidirectional LSTM-based network $LSTM_t$ to investigate how they relate to one another, as detailed in the following equation,
	\begin{equation}
		\small
		\begin{split} 
			\mathbf{\theta_k}= \frac{1}{| \mathbf{I}(l_i = k)|} \sum [s_i *  \mathbf{I}(l_i = k)] \\
			\mathbf{t}=LSTM_t(\mathbf{\theta_1},\mathbf{\theta_2},...,\mathbf{\theta_k};W,b)\\
		\end{split} 
	\end{equation}
	where $k$ is the category index, $| \mathbf{I}(l_i = k)|$ is the number of class $k$, $ \mathbf{I}()$ is one if the logical expression in the bracket is true, otherwise is zero.  
	

	\subsection{Knowledge Alignment}
	\label{knalign}
	We develop an effective loss function that encourages the alignment between model embedding $\mathbf{h}$ and task embedding $\mathbf{t}$ within our retrieval proxy, enabling knowledge-consistent model retrieval.
	As shown in Fig. \ref{framwork},
	a model embedding consistency function $L_{MKC}$ is used to encourage neural networks to overcome the noise caused by external datasets, and a spatial alignment loss function $L_{SAL}$ is used to overcome various biases between  $\mathbf{h}$ and $\mathbf{t}$.
	
	\textit{Model Embedding Consistency Loss.}
	Assuming  $\mathbf{h_i^{i'}}$ 
	is the generated embedding of $\Phi_i$ by using probe dataset $P_{i'}$,  and it contains both the model's inherent knowledge and noise from  $P_{i'}$. To address this, a category loss $L_{MKC}$ is used to incentivize $M_{EXT}$ to learn the knowledge specific to $\Phi_i$, using a distinct index $i$ for each model as the training label,
	\begin{equation}
		L_{MKC}=CE(\mathbf{h_i^{i'}},i)
	\end{equation} 
	where $CE$ is the well-established cross-entropy loss function\cite{croloss}. 
	
	\textit{Spatial Alignment Loss.}
	After vectorization, there are still semantic and mapping space bias between $\mathbf{h}$ and  $\mathbf{t}$. $\mathbf{h}$ is encoded from two types of samples, while $\mathbf{t}$ aggregates the features of each category, that leads to semantic differences. These differences in mapping space due to $M_{EXT}$ and $Q_{EXT}$ further contribute to the alignment biases. To suppress the biases, we characterize the knowledge consistency with cosine similarity  incorporating a margin of 0.4.
	\begin{equation}
		\small
		L_{SAL}(\mathbf{t_i},\mathbf{h_j}) =
		\begin{cases}
			1 - cos(\mathbf{t_i},\mathbf{h_j}), & \text{if } i = j \\
			\max(0, cos(\mathbf{t_i},\mathbf{h_j}) - \text{0.4}), & \text{else}
		\end{cases}
	\end{equation}
	
	where $i$ and $j$ are the indexes of the query task and candidate model, respectively, $cos$ is the cosine distance.  
	Therefore, the final objective function is defined as follows:
	\begin{equation}
		L= L_{MKC}+\alpha \cdot L_{SAL}
		\vspace{-0.1cm}
	\end{equation}
	where $\alpha$ is a is a constant parameter to balance the different losses, and it is set to 1 in our experiment.
	
	Finally in the well-established knowledge alignment space after training,
	the model index $j$ with the strongest semantic
	correlation between $\mathbf{t}$  and candidate model embeddings $\{\mathbf{h_i}\}_{i=1}^m$ can be obtained by the semantic measurement $\mathcal{DIS}()$, which is implemented by the cosine distance.
	\begin{equation}
		j=\mathop{\arg\min}\limits_{i} \mathcal{DIS}(\mathbf{t},\mathbf{h_i})
	\end{equation}
	
	\section{Experiments}
	We compare our Know2Vec with several state-of-the-art methods in two scenarios: NNR and SF-MTE. There are four groups of comparison methods.
	\label{sec:reference_examples}
	\begin{itemize}
		\item \textbf{Statistical SF-MTE methods:} H-Score~\cite{HScore}, NCE~\cite{NCE}, Leep~\cite{LEEP}, NLeep~\cite{NLEEP}, and LFC~\cite{LFC}.
		\item \textbf{Dynamic SF-MTE methods:} LogME~\cite{LogME} and Model Spider~\cite{Spider}.
		\item \textbf{General NNR methods:} TANS~\cite{TANS} and DNNR~\cite{DNNR}. Since DNNR requires to train a considerable neural network model for querying data, we do not compare with it.
		\item \textbf{Universal Large language models (LLMs)}:
		We also  examine GPT-4~\cite{gpt4} and Gemini~\cite{gemini} due to their powerful generation capability.
	\end{itemize}
	
	\begin{table}[!h]
		\centering
		\small
		\setlength{\tabcolsep}{1.2mm} 

		\label{haccnnr}
		\begin{tabular}{c|c|c|c|c|c|c}
				\hline
				\centering
				&R@1 &R@3 &V. Acc &Ft. Acc & Time &  Pri.\\
				\hline
				H-Score &3.02 &7.76&29.07&58.94 &23.21 &$\gamma$\\
				NCE &91.81  & $\mathbf{100}$ &94.03& 90.22 &10.09 & $\gamma$\\
				Leep & 93.10 &$\mathbf{100}$ &94.33 &91.66   &11.28 &$\gamma\gamma$\\
				NLeep&75.86 & 92.24 &83.84 &85.99  &10.60 &$\gamma$$\gamma$\\
				LFC &91.38 &$\mathbf{100}$ & 92.79 & 90.25&10.03 &$\gamma$$\gamma$ \\
				\hline
				LogME &50.43 &62.93 &64.68 & 77.30 &11.32 &$\gamma$$\gamma$\\
				Model Spider& 3.87&5.60&27.23&39.18& 4.28 &$\gamma$$\gamma$\\
				\hline
				TANS&82.75 & $\mathbf{100}$ & 93.70& 94.22 & $\mathbf{\leq 0.1}$ &$\gamma$ \\
				Ours & $\mathbf{94.82}$ &$\mathbf{100}$ &$\mathbf{94.87}$  &$\mathbf{95.67}$  & $\mathbf{\leq 0.1}$ &$\gamma$$\gamma$$\gamma$  \\
				\hline
				GPT-4&-&-&-& 46.37&34.48&$\gamma$$\gamma$$\gamma$\\
				Gemini& -&-&-& 33.87& 70.93&$\gamma$$\gamma$$\gamma$\\
				\hline
			\end{tabular}
			\caption{Performance comparison of NNR tasks. 
			}
			\label{haccnnr}
		\end{table}

		\begin{figure*}[!h]
			\label{visper}
			\centering
			\begin{minipage}{\linewidth}
				\begin{minipage}{0.14\linewidth}
					\centering
					\includegraphics[width=\textwidth]{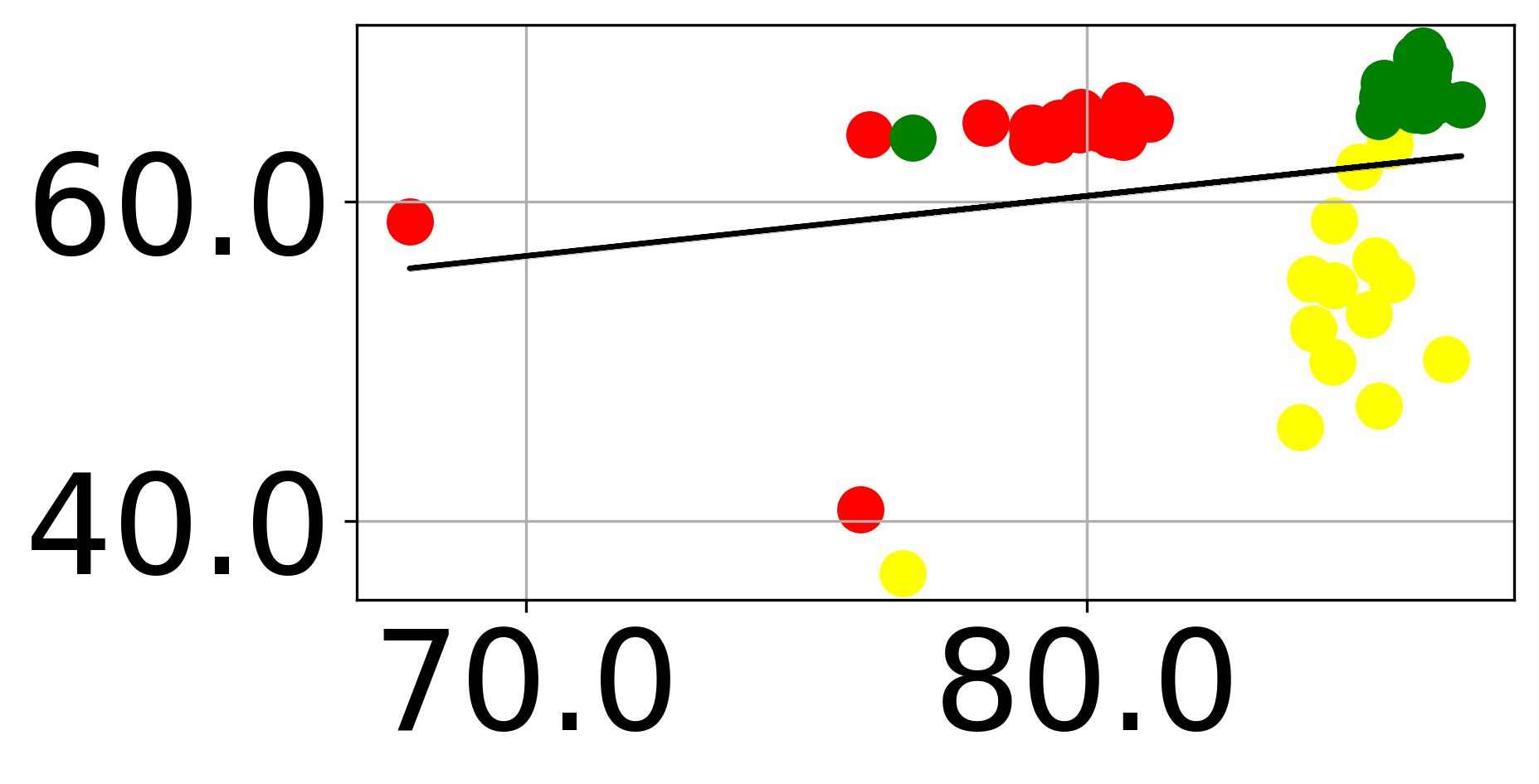}
					\vspace{0.2cm}
					\label{}
				\end{minipage}
				\begin{minipage}{0.14\linewidth}
					\centering
					\includegraphics[width=\textwidth]{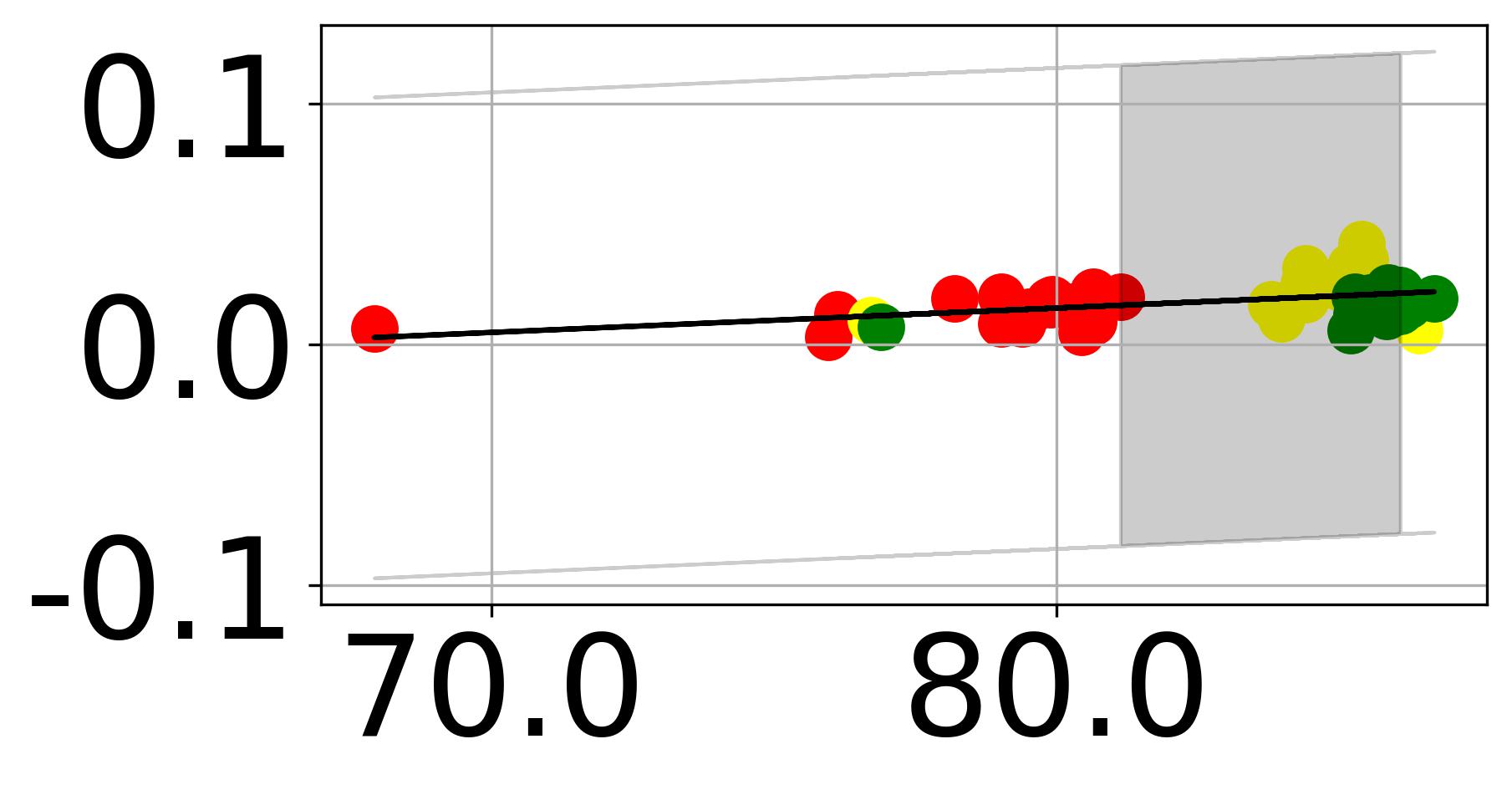}
					\vspace{0.2cm}
				\end{minipage}
				\begin{minipage}{0.14\linewidth}
					\centering
					\includegraphics[width=\textwidth]{
						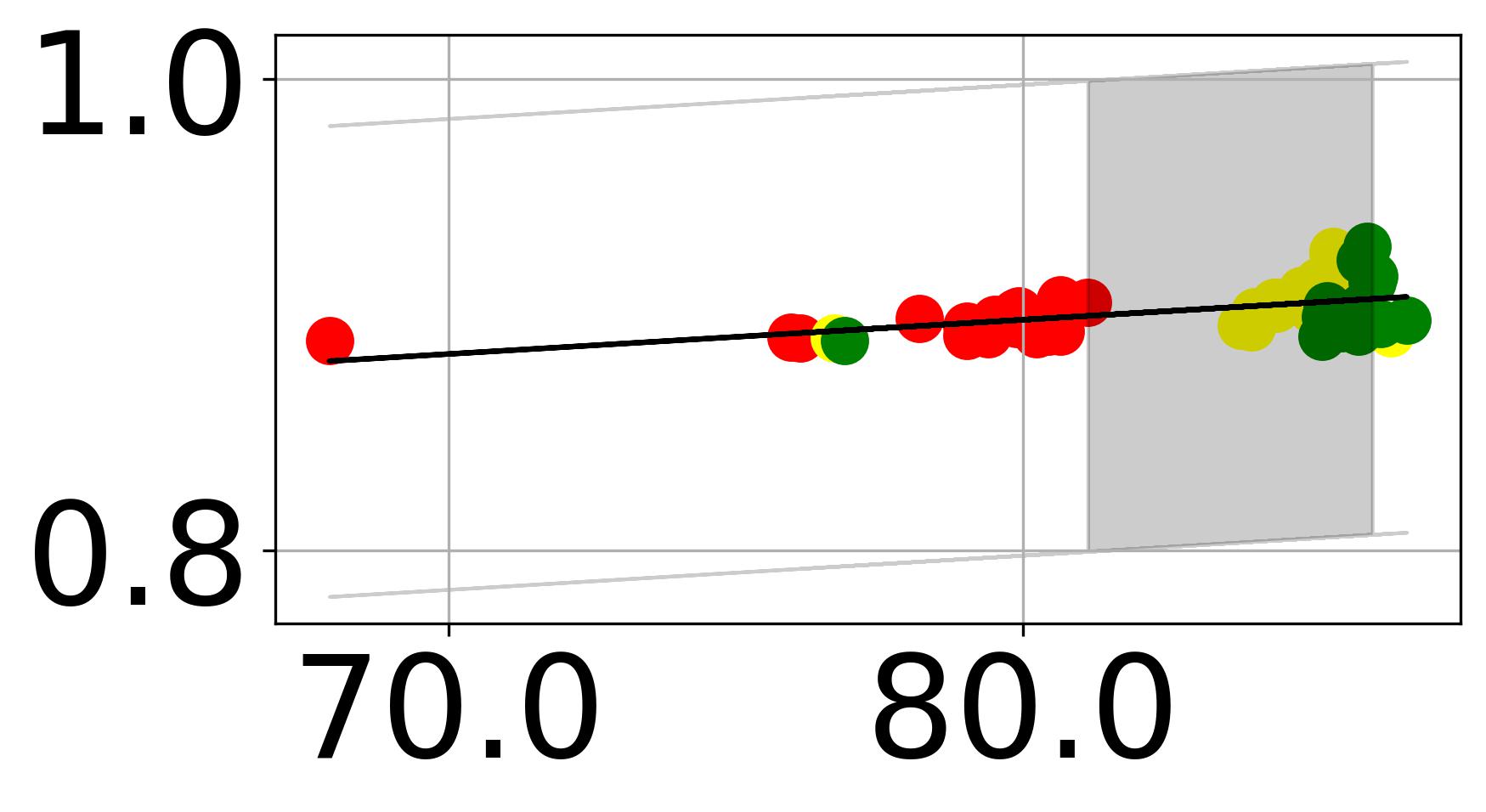
						}
						\vspace{0.2cm}
					\end{minipage}
					\begin{minipage}{0.14\linewidth}
						\centering
						\includegraphics[width=\textwidth]{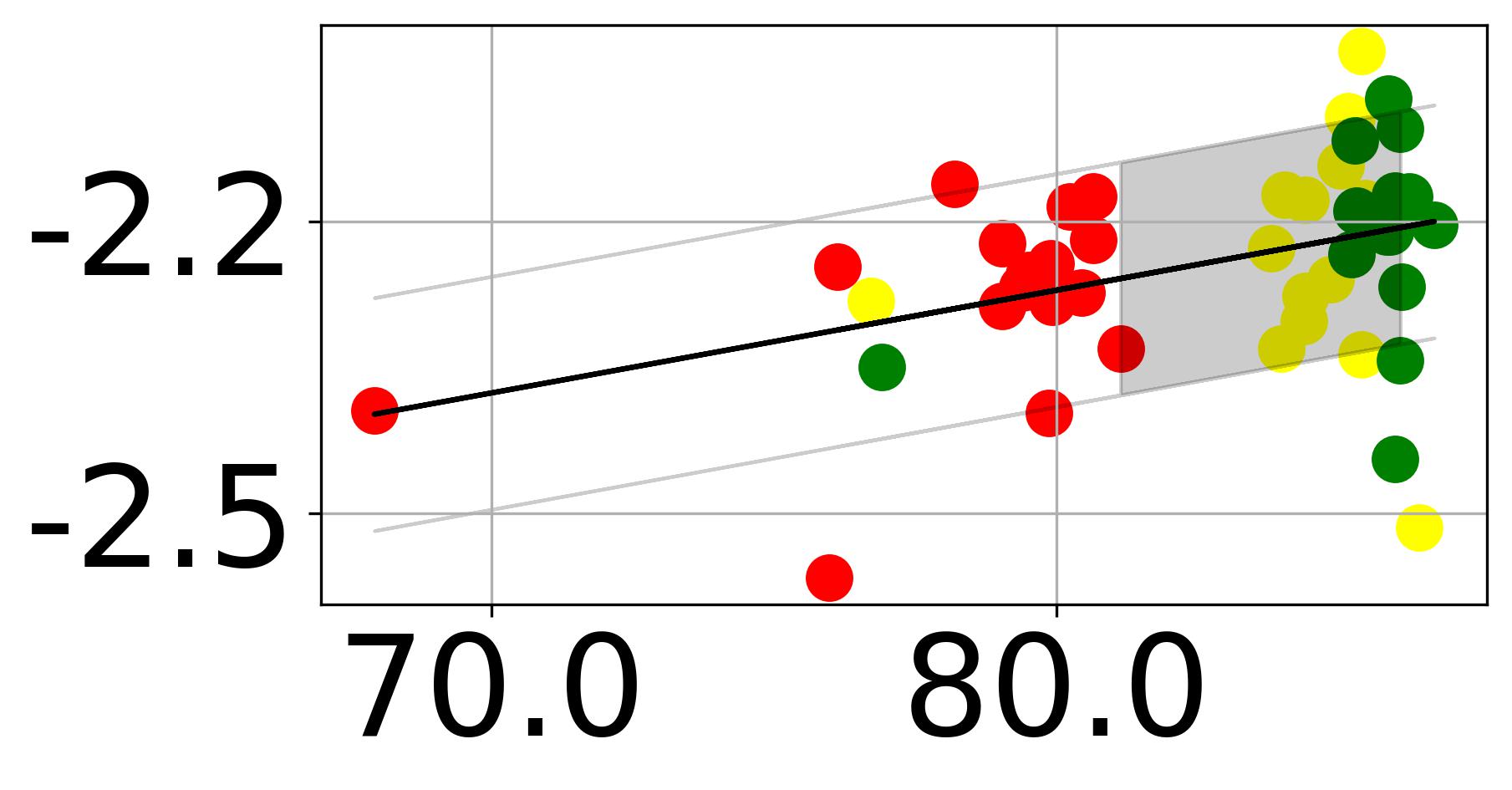
							}
							\vspace{0.2cm}
						\end{minipage}
						\begin{minipage}{0.14\linewidth}
							\centering
							\includegraphics[width=\textwidth]{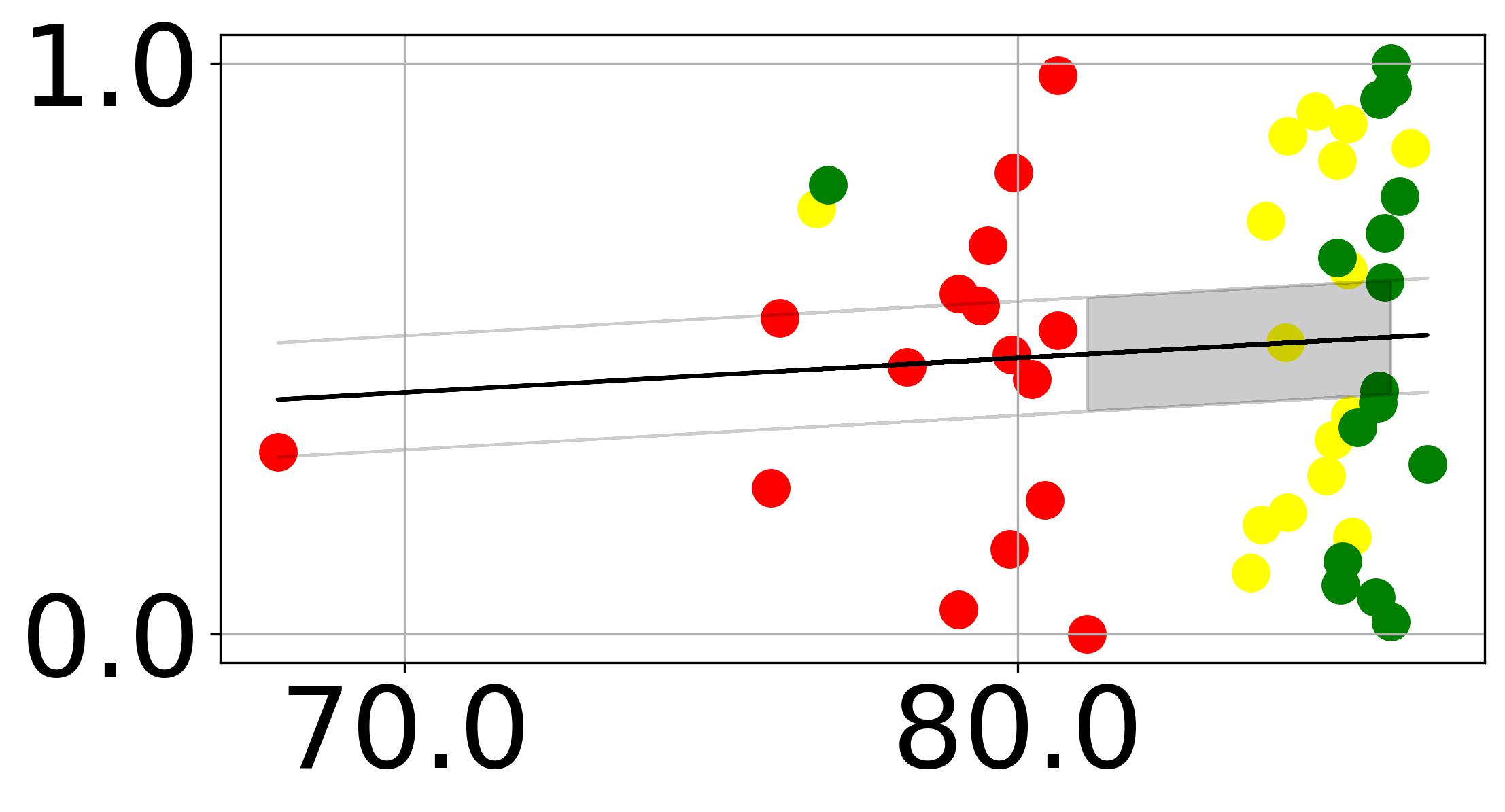}
							\vspace{0.2cm}
						\end{minipage}
						\begin{minipage}{0.14\linewidth}
							\centering
							\includegraphics[width=\textwidth]{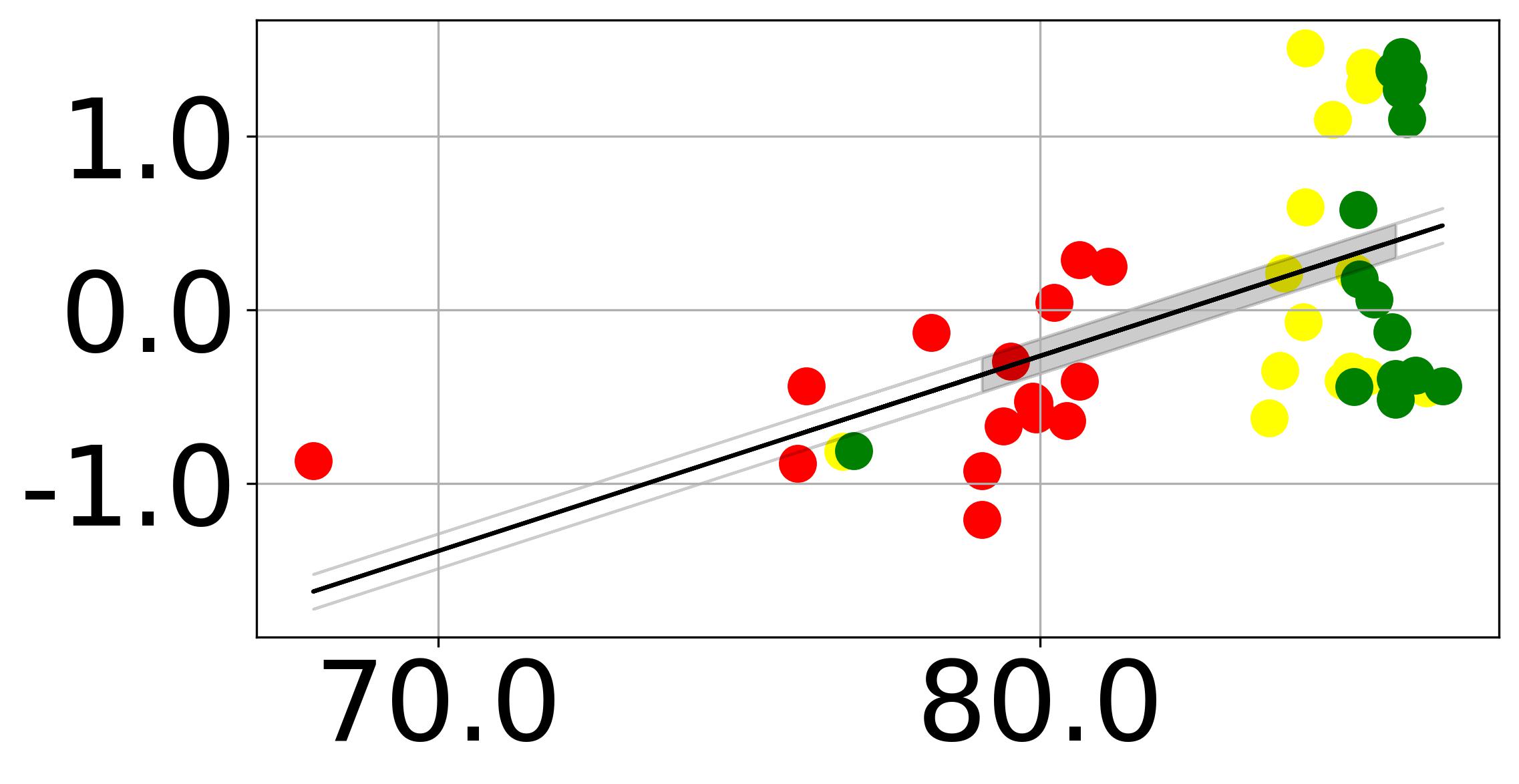}
							\vspace{0.2cm}
						\end{minipage}
						\begin{minipage}{0.13\linewidth}
							\centering
							\includegraphics[width=0.9\textwidth]{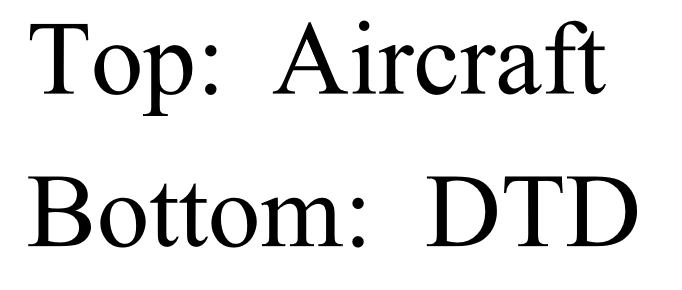}
							\vspace{0.5cm} 
						\end{minipage}
					\end{minipage}
					\begin{minipage}{\linewidth}
						\vspace{-0.6cm}
						\begin{minipage}{\linewidth}
							\begin{minipage}{0.14\linewidth}
								\centering
								\includegraphics[width=\textwidth]{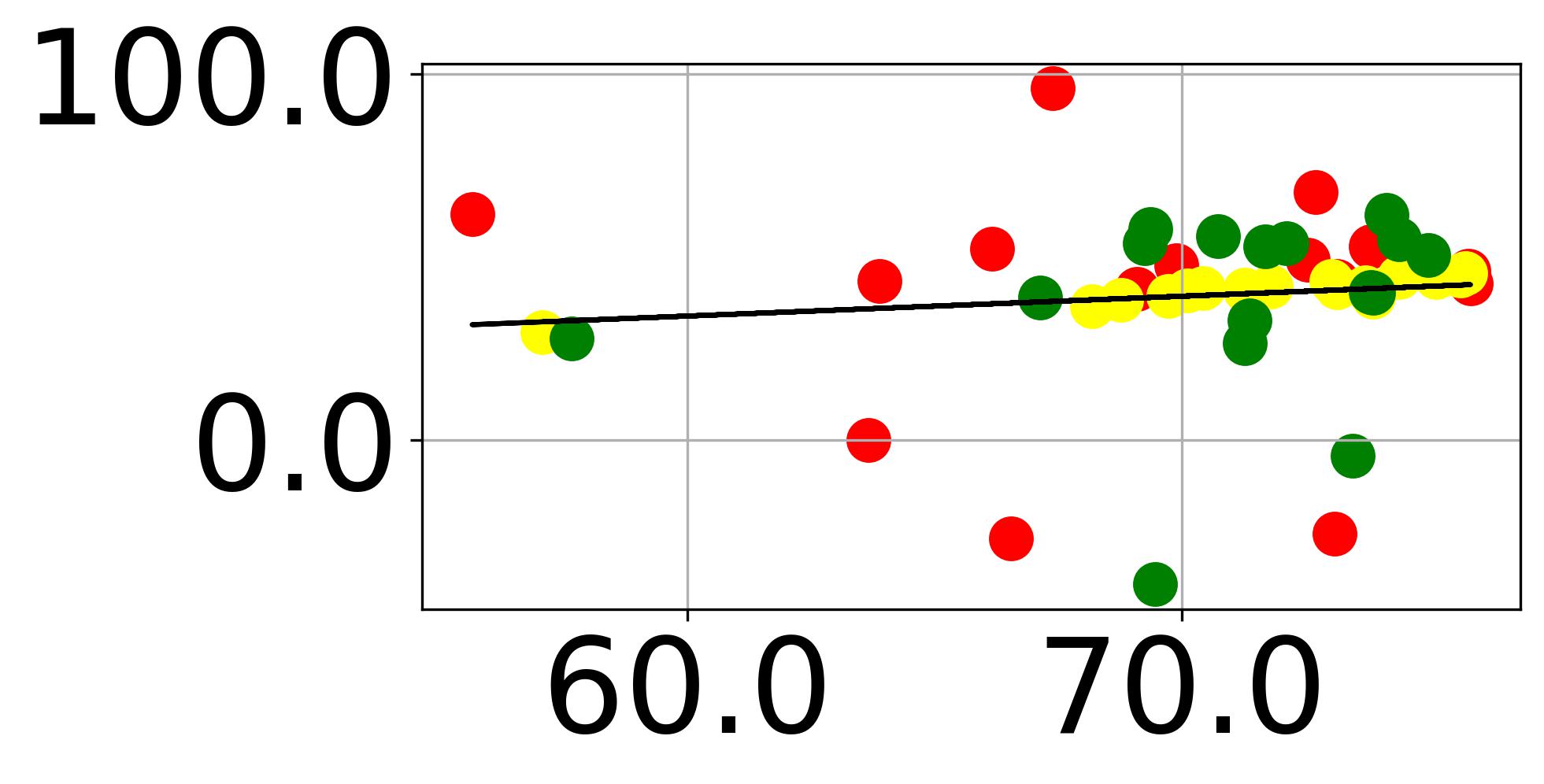}
								\captionsetup{font=small}
								\vspace{-0.6cm}
								\caption*{H-Score}
								\label{}
							\end{minipage}
							\begin{minipage}{0.14\linewidth}
								\centering
								\includegraphics[width=\textwidth]{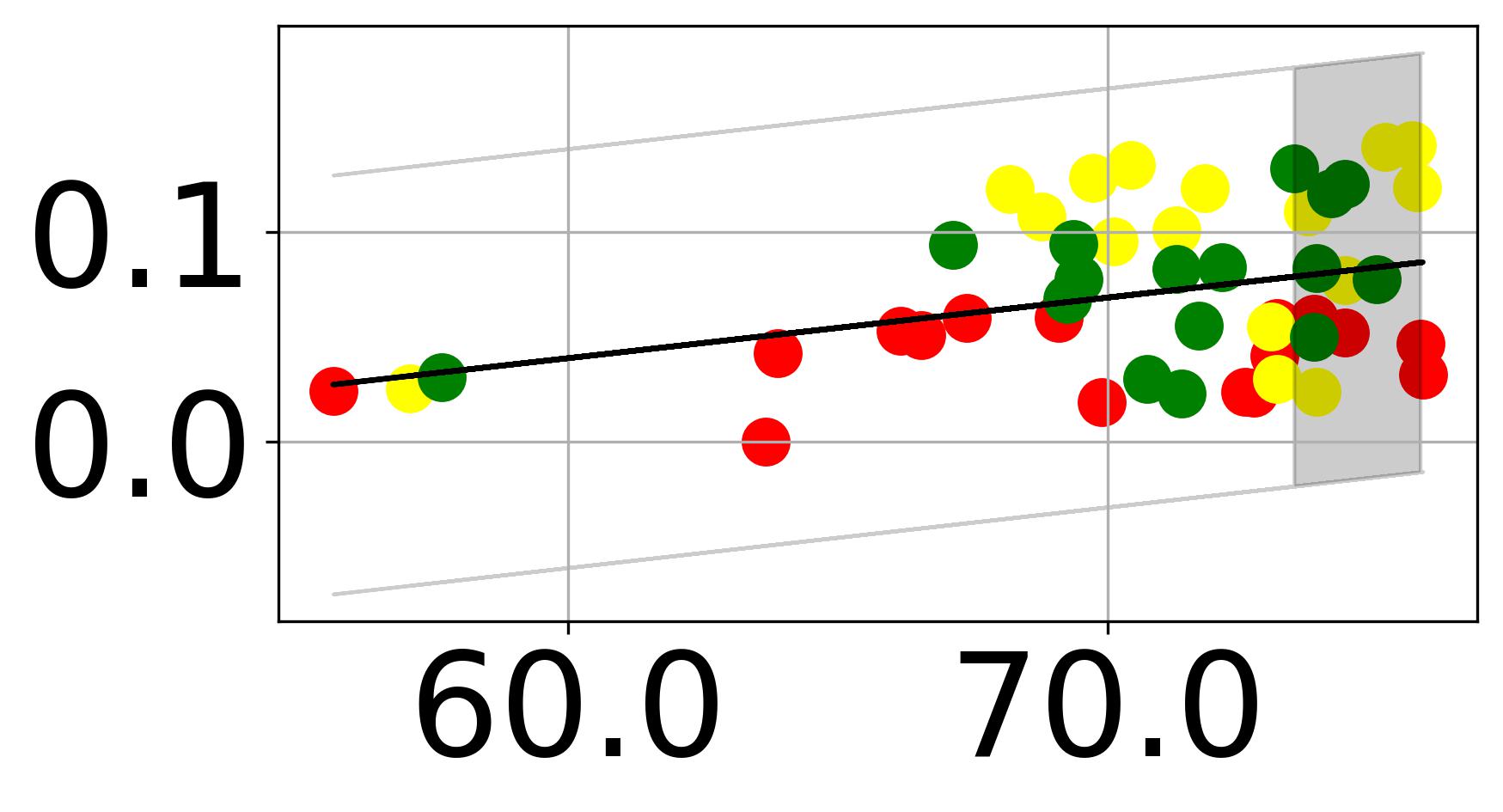}
								\captionsetup{font=small}
								\vspace{-0.6cm}
								\caption*{LFC}
							\end{minipage}
							\begin{minipage}{0.14\linewidth}
								\centering
								\includegraphics[width=\textwidth]{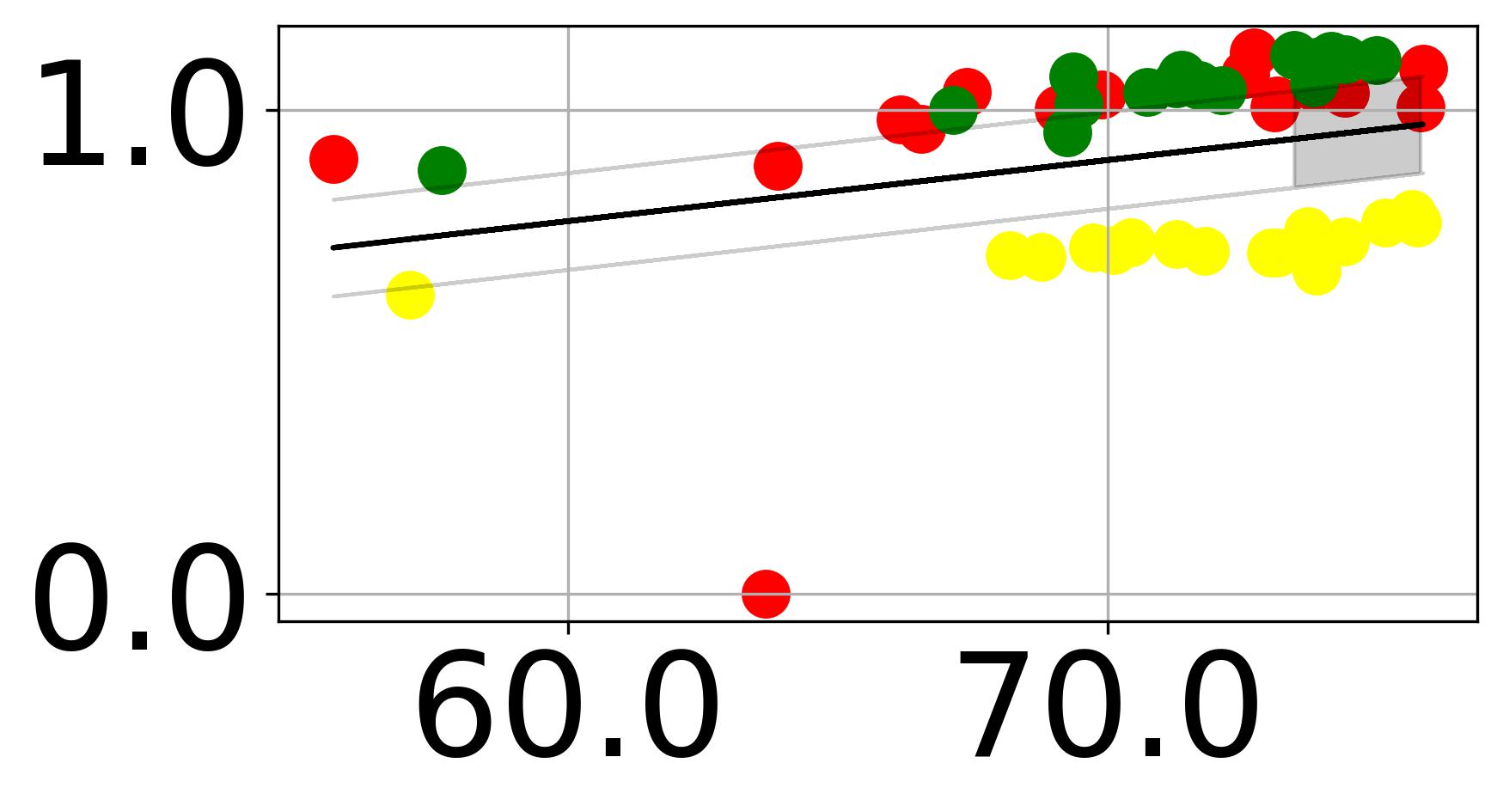
								}
								\captionsetup{font=small}
								\vspace{-0.6cm}
								\caption*{LogME}
							\end{minipage}
							\begin{minipage}{0.14\linewidth}
								\centering
								\includegraphics[width=\textwidth]{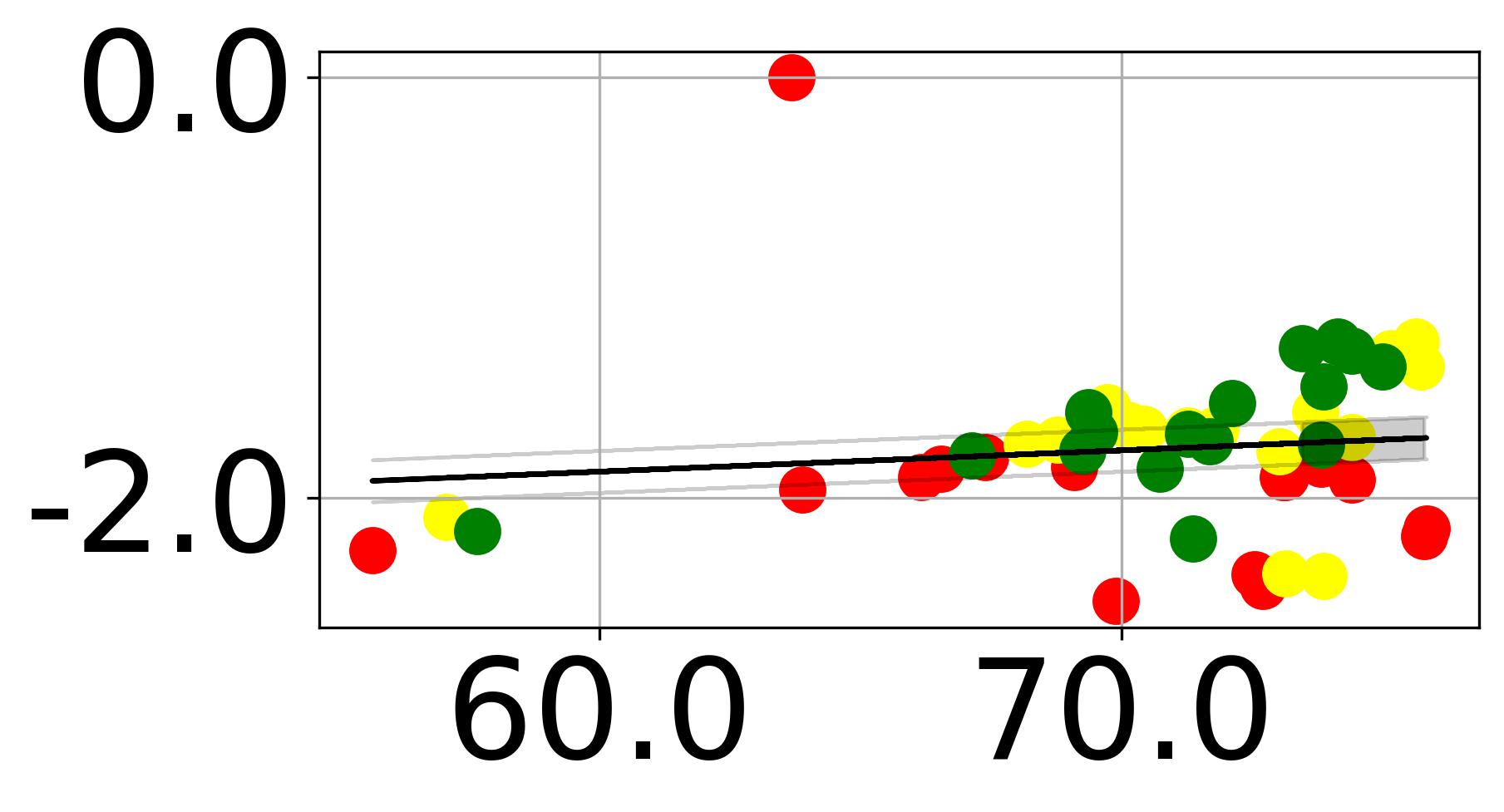
									}
									\captionsetup{font=small}
									\vspace{-0.6cm}
									\caption*{NLEEP}
								\end{minipage}
								\begin{minipage}{0.14\linewidth}
									\centering
									\includegraphics[width=\textwidth]{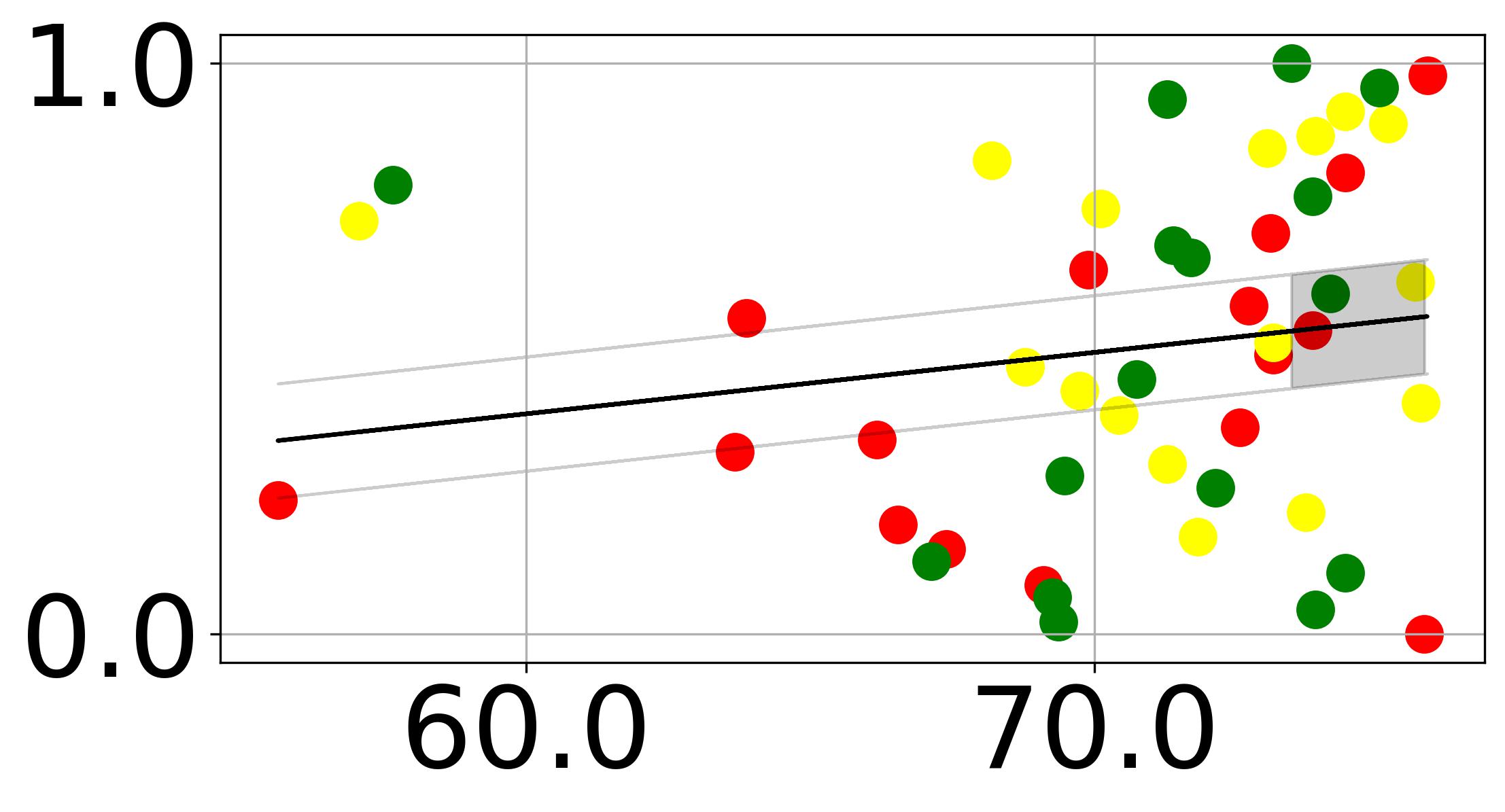}
									\captionsetup{font=small}
									\vspace{-0.6cm}
									\caption*{Model Spider}
								\end{minipage}
								\begin{minipage}{0.14\linewidth}
									\centering
									\includegraphics[width=\textwidth]{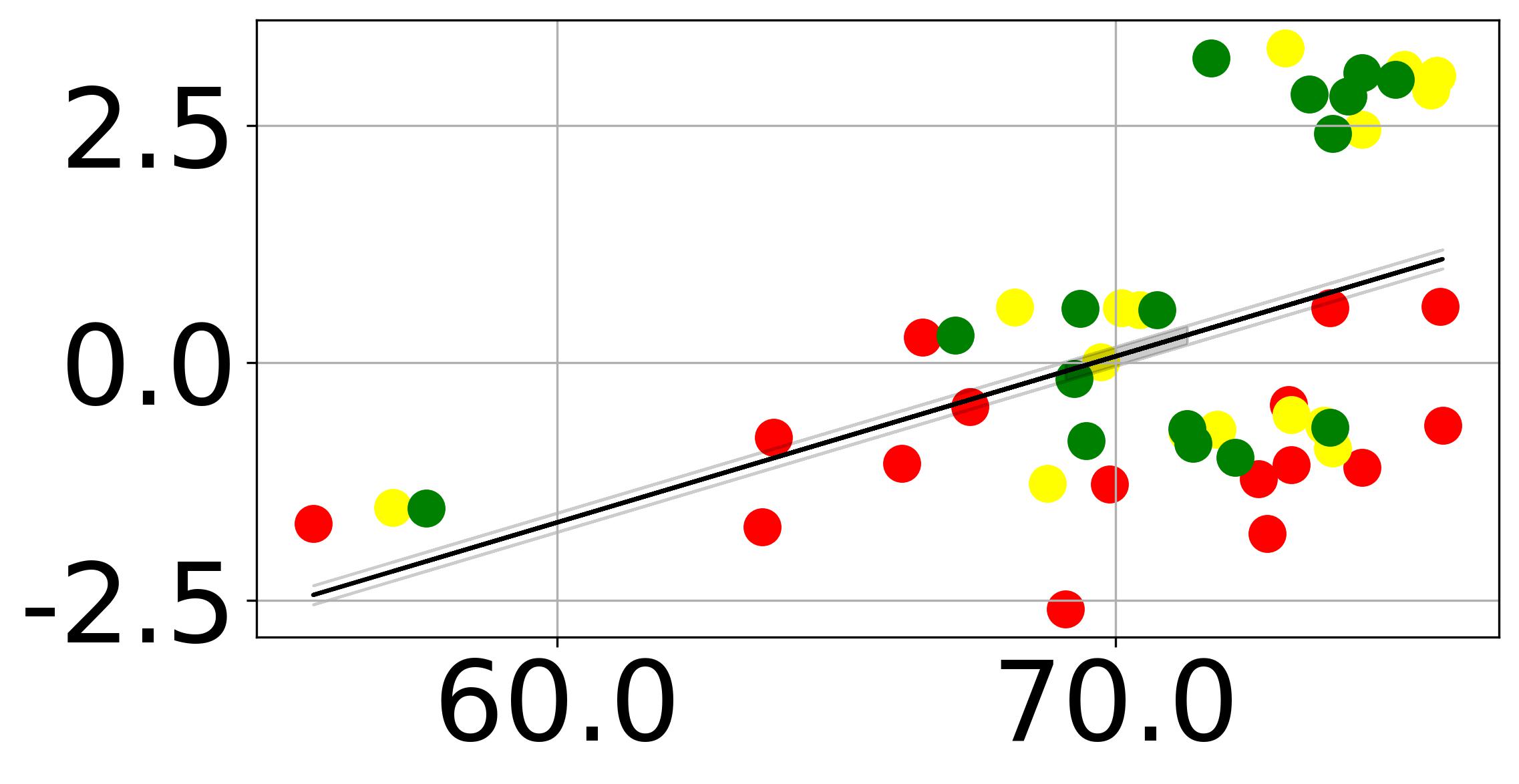}
									\vspace{-0.6cm}
									\caption*{Ours}
									\label{datakm1}
								\end{minipage}
								\begin{minipage}{0.13\linewidth}
									\centering
									\includegraphics[width=0.9\textwidth]{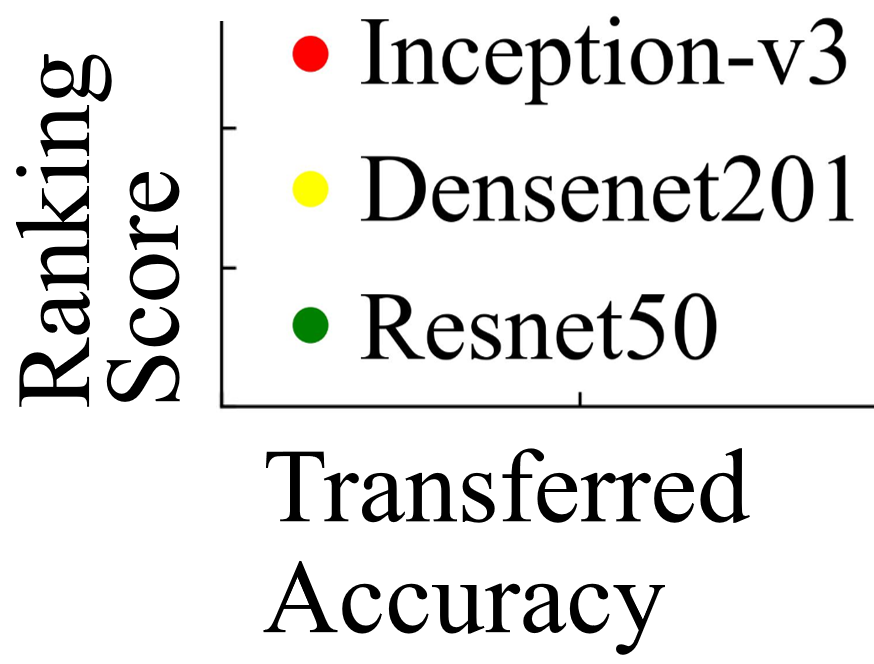}
								\end{minipage}
							\end{minipage}
						\end{minipage}
						
						\vspace{-0.1cm}
						\caption{Visualization description of  Pearson correlation on  SF-MTE experiments.
						}
						\label{visper}
					\end{figure*}
					
					\begin{table*}[t]
						\centering
						\small
						\label{hacc1}
						\begin{tabular}{c|c|c|c|c|c|c|c|c|c|c}
							\hline
							\centering
							&\multicolumn{4}{c}{Classification} &\multicolumn{4}{c}{Regression}&\multicolumn{2}{c}{
								Mean}
							\\		\hline
							\centering
							&\multicolumn{2}{c}{DTD} &\multicolumn{2}{c}{Aircraft}&\multicolumn{2}{c}{UTKFace} &\multicolumn{2}{c}{dSprites}  \\
							\hline
							& P. & S.  & P. &S.  & P. &S. & P. &S  & P. &S. \\
							\hline
							H-Score&0.1081 &0.2311  &0.1915 &0.4967  &  -0.0011 &-0.0012 & \textbf{0.2243}	&0.2014 & 0.0945& 0.2320\\
							NCE&-0.1650 &-0.2559 &0.0845 &-0.0229&-&-&-&- &-0.0402&-0.1394 \\
							Leep & -0.1672 & -0.2229 &-0.2079 &-0.1718 &-&-&-&-  & -0.1875&-0.1973 \\
							NLeep&0.1153 &0.2298  &0.3852 &0.3213& -&-&-&-&0.2502&0.2755\\
							LFC& 0.3508  & 0.2383 &0.4323 &0.4733 & -&-&-&-&\textbf{0.3915} &0.3558
							\\
							\hline
							LogME&0.2367  &0.3280 &0.5310 &0.5337 &-0.0038 &-0.0031 & 0.1082 & 0.1114&0.2180&0.2425 \\
							Model Spider & 0.1705 &  0.2937 &  0.0793 &0.1263  &0.1763 &0.1599 &-0.0365 & -0.0483 &  0.0974&0.1329\\
							\hline
							TANS& 0.2365&0.2738 & -0.3804& -0.3110 & -0.0057&	0.0091& 0.1054	&0.1126 & -0.0110&0.2112 \\
							Ours & \textbf{0.4942}&\textbf{ 0.5122}  & \textbf{0.5545}& \textbf{0.5779}  &\textbf{ 0.1900}  &\textbf{ 0.1909} & 0.1917 & \textbf{0.2608} & 0.3576 & \textbf{0.3854}\\
							\hline
							GPT-4& -0.3228 &-0.1030&   -0.2093& -0.0273&0.0475 &0.0412 &-0.0759 &-0.0814 &-0.1401& -0.0426 \\
							Gemini& -0.0099& 0.0828 & 0.0039 &-0.0184 &0.0269 &0.0764 &0.0797 &0.1475 &   0.0251&0.0720\\
							\hline
						\end{tabular}
						\caption{Performance comparison on source-free model transferability estimation tasks. }
						\label{hacc1}
					\end{table*}

					\subsection{Performance Comparison on NNR Tasks}
					\subsubsection{Experimental Setup.}
					\label{sec:retrieval}  The evaluation experiment is carried on a modified model-hub created from Kaggle\footnote{\url{https://www.kaggle.com/}} with diverse real-world datasets/models following the methodology outlined in TANS\cite{TANS}. We developed 58 classification models and 232 distinct testing tasks. The probe images are randomly selected from the training dataset of Know2Vec.  For fairness, we fine-tuned the model generated from LLMs for 500 steps, as LLMs typically generate neural network rather than select them.  In the specific vector computation, the vectors of our approach and TANS method are 256 dimensions long, while that of Model Spider is 1024 dimensions long.

					
					To thoroughly assess our method against benchmarks, we adopt a suite of established metrics, including:
					(1) Top-k hitting ratio (R@k,\%),
					that measures the overlap percentage between the top-k prediction results and the ground truth. (2) Valid Accuracy(V.Acc,\%) and Fine-Tuned Accuracy(Ft.Acc,\%), which quantify the accuracy of the query task on the top-1 selected model and the results after fine-tuning over 50 trials, respectively. (3) Search Time(Time, s). (4) Privacy(Pri.). We categorize privacy into three tiers of model access  permissions: white-box access $\gamma$, grey-box access $\gamma\gamma$ and black-box access $\gamma\gamma\gamma$. 
					
					\subsubsection{Experimental Analysis.}
					The quantitative comparison results of NNR task are shown in Table \ref{haccnnr}. The best score is in bold. As can be seen,
					in the evaluation of statical methods, Leep achieves higher retrieval accuracy among the evaluated methods, due to its focus on average loglikelihood. However, it falls slightly behind in search time compared to NCE, NLeep, and LFC. H-Score, unfortunately, underperforms in both search time and accuracy, possibly due to its complex calculation and lack of consideration for similarities within categories. Dynamic methods such as Model Spider and LogME also struggle, possibly because of their focus on ranking order of transferability on abundant downstream data, whereas NNR task is more concerned with the performance of the selected top-1 model. Despite its lower accuracy than LogME, Model Spider benefits from vector-based computations, consuming less search time. The same beneficiaries also include our method and TANS.
					Fortunately, TANS excels in search time and provides substantial retrieval accuracy although it offers only a sub-optimal level of privacy. Our method, while maintains superiority in terms of computation time, offers superior retrieval performance and maintains privacy. Notably, our method achieved a 1.72\% increase in retrieval accuracy over the suboptimal result, which demonstrates the superior precision of the knowledge alignment space embedded in our proposed proxy. Undoubtedly, although GPT-4 and Gemini ensure a strong privacy since they do not require access to model zoo, their performances in statistical data falls short of expectations.
					
					\label{isomorphic}
					\subsection{Performance Comparison  on SF-MTE tasks}

					\subsubsection{Experimental Setup.}
					We construct a heterogeneous model zoo similar to previous work\cite{Spider}, where we collect 48 publicly available pre-trained models trained on diverse datasets\footnote{\url{https://bmwu.cloud//}}, covering various neural network architectures. The probe dataset is filtered from several publicly available datasets. We evaluate various methods on 4 different downstream tasks,  Aircraft\cite{aircraft} and DTD\cite{dtd}  for classification, UTKFace \cite{UTKFace} and dSprites\cite{dsprites} for regression. We leave blank for the regression column of NCE, Leep, NLeep, and LFC since they cannot be used for regression tasks. 


					
					
					We measure the performance of SF-MTE with  Pearson(P.)\cite{pearson} and Spearman(S.)\cite{spearman} correlation scores, as they are widely adopted \cite{LEEP,NLEEP,Spider,TANS,LFC} to evaluate the relationship between the predicted transferability scores and test accuracy.
					
					\subsubsection{Experimental Analysis}
					
					\begin{figure*}[t]
						\vspace{-0.4cm}
						\centering
						\hfill
						\begin{minipage}{0.24\textwidth}
							\centering
							\begin{table}[H]
								\small
								\setlength{\tabcolsep}{1mm}
								\begin{tabular}{c|c|c|c}
									\hline
									$M_{EXT}$&\multicolumn{3}{c}{$Q_{EXT}$ } \\
									\hline
									& LSTM & ConCat & Avg. \\
									\hline
									LSTM & $\mathbf{94.82}$ & 92.54 & 88.14 \\
									ConCat & 90.87 & 94.05 & 87.37 \\
									Avg. & 90.33 & 89.34 & 89.35 \\
									\hline
								\end{tabular}
								\caption{Ablation study of retrieval architecture.}
								\label{ab1}
							\end{table}
						\end{minipage}%
						\hfill
						\begin{minipage}{0.24\textwidth}
							\centering
							\begin{table}[H]
								\small
								\centering
								\setlength{\tabcolsep}{1mm}
								\begin{tabular}{c|c}
									\hline
									Loss Function & Accuracy \\
									\hline
									w/o $L_{MKC}$ & 82.97 \\
									$L_{SAL}$(Cos.) & $\mathbf{94.82}$ \\
									$L_{SAL}$(Con.) & 93.53 \\
									\hline
									$P_{train}$ & 95.25 \\
									\hline
								\end{tabular}
								\caption{Ablation study of loss functions and probe dataset.}
								\label{loss}
							\end{table}
						\end{minipage}%
						\hfill
						\begin{minipage}{0.45\textwidth}
							\centering
							\includegraphics[width=0.78\linewidth]{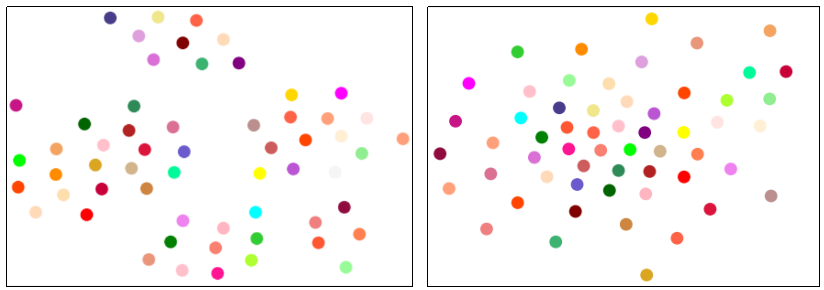}
							\caption{Visual description of model knowledge vectors.}
							\label{datakn1}
							\vspace{-0.4cm}
						\end{minipage}
					\end{figure*}
					
					The statistical evaluations of model transferability over classification and regression tasks are shown in Table \ref{hacc1}.
					To provide a clear representation of the correlation between the baseline predictions and the actual accuracy, we visualize the top-6 results of Pearson correlation scores in Fig.~\ref{visper}. In the static methods, LFC and NLeep show a consistent performance, achieving positive and satisfactory scores in both Pearson and Spearman evaluations. By comparison, NCE and Leep show negative correlation coefficients.  H-Score performs better on Spearman score than Pearson score, probably because Pearson score is more sensitive to the predicted outlier's scores. 
					In dynamic SF-MTE methods, Model Spider performs poorly, perhaps because its insufficient robustness. In contrast, LogME excels in  classification tasks, a testament to the precision of its linear estimation model. TANS struggles in the SF-MTE tasks, while GPT-4 and Gemini display negative or near-zero correlation coefficients on most datasets, indicating a less competitive performance compared to other methods.
					Our method excels in classification tasks, with improvements of Spearman coefficient reaching 0.1842 and 0.0442 over the sub-optimal result. This indicates a strong correlation between predicted transferability scores and actual test accuracy, as shown in Fig.\ref{datakn1}, thanks to the semantically rich knowledge vectors  and precise matching process. Although our method is not the best in every evaluation dimensions, which may due to being trained only on classification tasks, it still delivers satisfactory results across all tested tasks.

					\subsection{Ablation Study }
					We assess the performance of each component designed in the proposed proxy on the Kaggle-hub. 
					
					\label{ablationarch}
					
					\subsubsection{Analysis of Knowledge Vectorization Architecture.}
					As shown in Table \ref{ab1}, we explored alternative sequence encoding methods, transitioning from an LSTM network to simpler network such as averaging (Avg.) and concatenation (ConCat).
					Fixing the structure of $M_{EXT}$ to an LSTM, we found that the result in the first column (94.82\%) is significantly higher than those in the second column (92.54\%)  and third column (88.14\%), suggesting that the $LSTM_t$ in $Q_{EXT}$ captures query knowledge more accurately.  Likewise, the first row's retrieval accuracy in the first column significantly outperforms the other rows, highlighting the effectiveness of LSTM-based vector encoder in $M_{EXT}$ in extracting detailed model information.
					We further made T-SNE visualization of model representations before(left) and after(right) encoding in $M_{EXT}$. In Fig.~\ref{datakn1}, different colors correspond to the knowledge vectors for different models. The right figure shows an improvement over the left by correctly separating models that were incorrectly clustered together based on their semantics.
					
					\subsubsection{Analysis of Different Loss Functions.}
					First, we tested two unsupervised loss functions, cosine loss (Cos.) and contrastive loss (Con.) for spatial alignment loss $L_{SAL}$. It can be seen from Table \ref{loss}, Know2Vec achieved the highest accuracy of 94.82\% with $L_{SAL}(\text{Cos})$, allowing for the natural knowledge alignment. Moreover, we observe a slight drop in performance without $L_{MKC}$, and this highlights the importance of $L_{MKC}$ in filtering noise from model vectors. 
					\subsubsection{Analysis of Different Probe Datasets.}
					
					In Table \ref{loss}, the value of $P_{train}$ indicates the model retrieval accuracy when the target model's training dataset is used as the probe dataset, suggesting that alternative dataset might be as effective as the training dataset in generating model knowledge vectors.  The same image can serve as a probe to further generate knowledge vectors for different models with alternative dataset, thereby accurately depicting the semantic differences of models in the knowledge consistency space.
					
					
					\section{Conclusion}
					In this paper, we propose Know2Vec, a novel proxy for neural network retrieval under a black-box situation. This proxy translates both model knowledge and query data knowledge into vectors, and thus enhancing the accuracy of the retrieval process by ensuring the knowledge consistency among them.
					The experimental results from NNR and SF-MTE tasks confirm that Know2Vec surpasses the state-of-the-art baseline methods in retrieval precision with acceptable retrieval speed, while also addressing privacy concerns.
					

\appendix
\section{Detailed derivation and proof for Lemma 1}

\label{appproof}
\begin{theorem}[Mean Value Theorem]
	\label{theorem1}
	Let $f$ be a continuous function on the closed interval $[a, b]$ and differentiable on the open interval $(a, b)$. Then, there exists at least one point $c \in (a, b)$ such that
	\[ f'(c) = \frac{f(b) - f(a)}{b - a}. \]
\end{theorem}

\newtheorem{assumption}{Assumption}
\begin{assumption}
	For a binary classification function $\Phi(x)=\delta (\mathbf{w}*x+\mathbf{b})$, let $\delta = g_A - g_B$ be a differentiable function near the boundary decision sample.   This is a strong assumption since $g_A$ and $g_B$ are differentiable functions around the points of interest. It is also a simplification since the actual derivative of the ReLU function is not defined at zero.
\end{assumption}

\begin{lemma}
	The perturbation vectors in $KRM$ can also be obtained from the models with associating the external  datasets.
\end{lemma}
\begin{proof}
	We assume that $\delta = g_A - g_B$ can be locally approximated by a differentiable function $\hat{\delta}$, in the vicinity of the points $x_a^b$, $x_a$, $x_b$, $z_a^b$, $z_a$ and $z_b$.
	Applying Theroem \ref{theorem1}, there exists a point $x_c$ such that:
	\[
	\hat{\delta}'(\mathbf{w} \cdot x_c + \mathbf{b}) = \frac{\hat{\delta}(\mathbf{w} \cdot z_a^b + \mathbf{b}) - \hat{\delta}(\mathbf{w} \cdot x_a^b + \mathbf{b})}{\mathbf{w} \cdot (z_a^b - x_a^b)}
	\]
	Since $\hat{\delta}(\mathbf{w} \cdot x_a^b + \mathbf{b}) = 0$ and $\hat{\delta}(\mathbf{w} \cdot z_a^b + \mathbf{b}) = \lambda_3$,
	
	we have:
	\[
	\hat{\delta}'(\mathbf{w} \cdot x_c + \mathbf{b}) = \frac{\lambda_3}{\mathbf{w} \cdot (z_a^b - x_a^b)}
	\]
	Let's define $\sigma = z_a^b - x_a^b$. Then, we can rewrite the above equation as:
	\[
	\hat{\delta}'(\mathbf{w} \cdot (x_c + \sigma) + \mathbf{b}) = \frac{\lambda_3}{\mathbf{w} \cdot \sigma}
	\]
	Therefore, there exists $\sigma $ that satisfies $z_a^b=x_a^b +\sigma$.
	
	The proven conclusions can also be applied to other classification models, as they can be seen as a combination of binary classification models.
\end{proof}

\section{Experimental Setups and Implementation Details}
\label{detailexp}
\subsection{Implementation Details of the Kaggle Model Zoo.}
We fine-tuned the Mobile-Net\cite{mobilev2} model, which was pre-trained on ImageNet-1K \cite{imagenet}, and modified its architecture to cater for the needs of a classification model. Utilizing the Adam optimizer with a learning rate of \(10^{-4}\), all experiments were conducted on two NVIDIA TITAN Xp GPUs. To ensure a fair comparison, each model's training dataset is comprised over 1000 images. We implemented a random sampling strategy with a 9:1 ratio for training and validation datasets, guaranteeing that the classification testing accuracy of each model met the desired outcomes.

Our model zoo is a diverse collection that spans  various domains such as fruit classification, digital recognition, traffic sign detection,  medical image classification and others.  
The specific models are shown in Table \ref{lagglezoo}, which also  includes a brief description of each model's expertise domain. Due to the space constraint, the list has been truncated and does not include all model names. Please refer to TANS \cite{TANS} for the detailed information for each dataset.  
\begin{table}[h]
	\centering
	\begin{tabular}{c|c} 
		\hline
		\textbf{Expert Field} & \textbf{Model Name} \\
		\hline
		Fruit Recognition & fruit-recognition\_ch8.pth,  ... \\
		Anime Character &  simpsons4\_20\_394.pth, ... \\
		Traffic Sign &  gtsrb-german-traffic0.pth, ... \\
		Medical Image  & csep546-aut19-kc2\_0\_0.pth, ... \\
		Language Recognition &  devanagari-character7.pth, ... \\
		Landscape image &land-cover-class\_0\_18.pth...\\
		Digital Classification &synthetic-digits\_pra8.pth\\
		Digital Clock    &csep546-aut19-kc2\_0\_0.pth\\
		Gesture Classification& asl-alphabet\_grasskn19.pth\\
		Others & numta\_BengaliAI\_0\_101.pth,...
		\\
		\hline
	\end{tabular}
	\caption{List of model names in kaggle model zoo.}
	\label{lagglezoo}
\end{table}

\subsection{Implementation Details of the Pre-trained Model (PTM) Zoo.}
Following Model Spider \cite{Spider}, we construct a large model zoo
with 48 heterogeneous models. PTMs are pre-trained on 3 similar architectures, i.e., Inception\_V3 \cite{incepv3}, ResNet50 \cite{res50} and DenseNet201 \cite{densenet} from 16 datasets in different domains, including Caltech101\cite{caltech101},
Cars \cite{cars}, CIFAR10 \cite{cifar10}, CIFAR100 \cite{cifar100}, AID \cite{AID}, SUN397 \cite{SUN397}, Dogs \cite{standogs}, EuroSAT \cite{helber2019eurosat}, Flowers \cite{nilsback2008automated},
Food \cite{bossard2014food}, NABirds \cite{van2015building}, PACS \cite{li2017deeper}, Resisc45 \cite{resisc45}, SmallNORB \cite{smallnorb} and SVHN \cite{svhn}, STL10 \cite{stl10}. The detailed specific expert fields are shown in Table \ref{ptmzoo}.  

\begin{table}[h]
	\centering
	\begin{tabular}{c|c} 
		\hline
		\textbf{  Expert Field} & \textbf{Model Name} \\
		\hline
		Animals &Dogs,NABirds \\
		Plants &  Flowers \\
		Scene-based &  SUN397 \\
		Remote Sensing  & AID,Resisc45, EuroSAT \\
		3D objects & CIFAR100, Caltech101,CIFAR10\\
		&Food,Cars,SmallNORB\\
		Domain Adaption& PACS\\
		Others & SVHN,STL10\\
		\hline
	\end{tabular}
	\caption{List of model names in pretrained model zoo.}
	\label{ptmzoo}
\end{table}

\subsection{Training Details of Know2Vec}
For the Know2Vec training process of NNR task, we set the batch size to 200 and  minimize the training loss with the learning rate of 1e-4 on Adam optimizer. We conducted rigorous validation checks to ascertain that there was no dataset overlap among the model training datasets, Know2Vec the training datasets, and query task sets. We sample training and testing samples from datasets with similar data distributions to the model. During the construction process of the query task, 5 images are randomly sampled for each category, and each image is cropped to 64x64. Specifically, the hidden embedding size in all LSTM network are set to 1000. To quicken training process, the input samples $\{x_a,x_a^b,..., x_d \}$ of $M_{EXT}$ and $\{s\}$ of $Q_{EXT}$  are first  converted to the high-dimensional feature with a Resnet-18 \cite{targ2016resnet}  that is pretrained on ImageNet-1K \cite{imagenet} serve as the backbone neural network. 

For the training process of Know2Vec for SF-MTE task, we maintain a batch size of 200 and  minimize the training loss with the 1e-3 learning rate on Adam optimizer. Specifically, we utilized PyTorch's ExponentialLR scheduler, which applies a multiplicative factor to the current learning rate after every epoch, which is set to 0.95.
The task was trained from a diverse array of datasets,
including CUB2011\cite{cub2011},  CIFAR100 \cite{cifar100}, SUN397 \cite{SUN397}, Dogs \cite{standogs}, 'VLCS', ImageNet \cite{imagenet},
and VLCS \cite{vlcs}. To compose a representative training dataset, we randomly sample over 700 tasks with each task spanning 1 to 2 mixed datasets. Each task comprises 50 to 200 images, including 50-100 categories. Keeping consistent with Model Spider \cite{Spider}, each image is cropped to 224x224. 

We adopted a simplified approach to obtain the ground truth rankings for Know2Vec training task.  While it was feasible to collect a sufficient image dataset for each training task, the process of fine-tuning all of the 48 models across over 700 tasks   would have been prohibitively resource-intensive. 
For each model, we froze the feature extraction part and only fine-tuned the classification layer. These results served as  an approximation of the  comprehensive fine-tuned accuracy. The underlying intuition is that for a given training task $\mathbf{t}$, if model $\Phi_1$  performs significantly better than model $\Phi_2$ after full fine-tuning,  it is reasonable to expect $\Phi_1$ to continue this superior performance trend even with partial fine-tuning. Consequently, model pairs exhibiting substantial performance disparities on the same task tend to have a more pronounced impact on the loss functions during training process, thereby encouraging the retrieval framework to learn accurate matching knowledge.


\section{Detailed Experimental Analysis from Different Dimensions}

\subsection{Detailed Analysis of Influence Factors in Query Representation}
\subsubsection{ Robustness Analysis of Various Query Parameters}
\begin{figure}[!h]
	\centering
	\begin{minipage}{0.99\linewidth}
		\centering
		\includegraphics[width=1.00\linewidth]{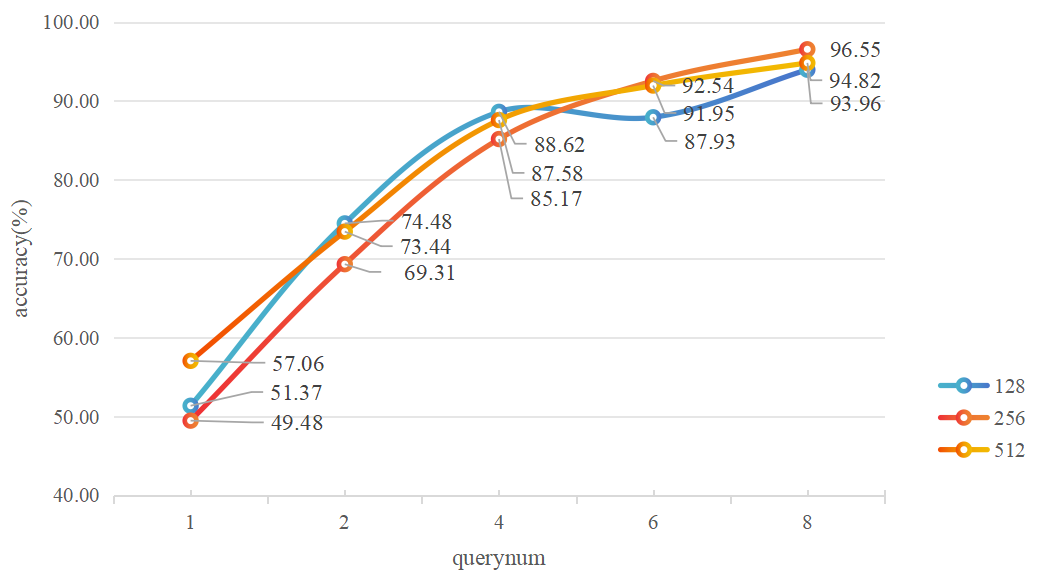}
	\end{minipage}
	\caption{Retrieval accuracy of different number of query samples.}
	\label{qpara}
\end{figure}

\begin{figure*}[!h]
	\label{datakn}
	\begin{minipage}{0.24\linewidth}
		\centering
		
		\includegraphics[width=0.99\linewidth]{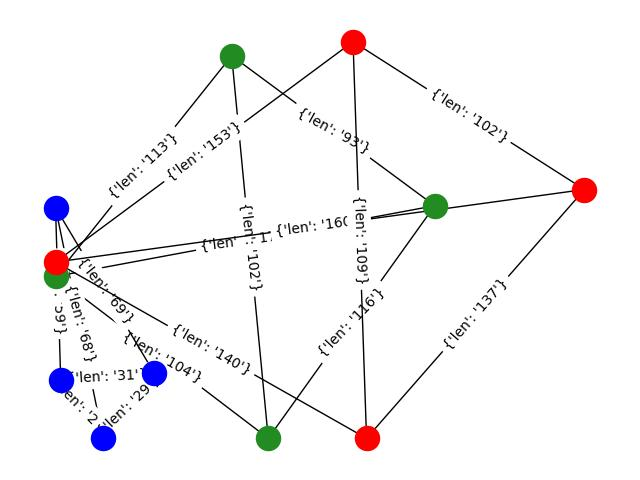}
		\caption*{(a) Total distance graph}
		\label{totdatakm}
	\end{minipage}
	\begin{minipage}{0.24\linewidth}
		\centering
		\includegraphics[width=0.99\linewidth]{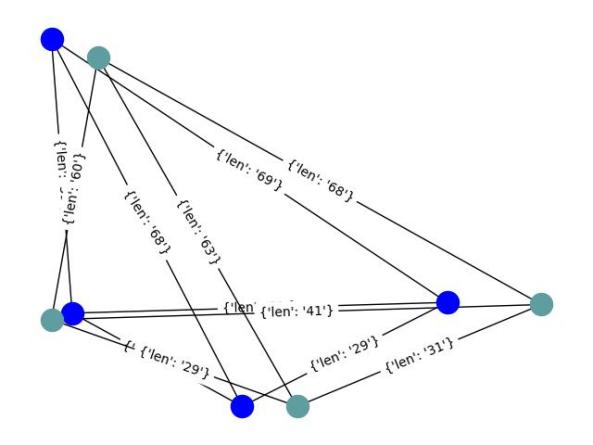}
		\caption*{(b) Devanagari-character7}
		\label{datakm1}
	\end{minipage}
	\begin{minipage}{0.24\linewidth}
		\centering
		\includegraphics[width=0.99\linewidth]{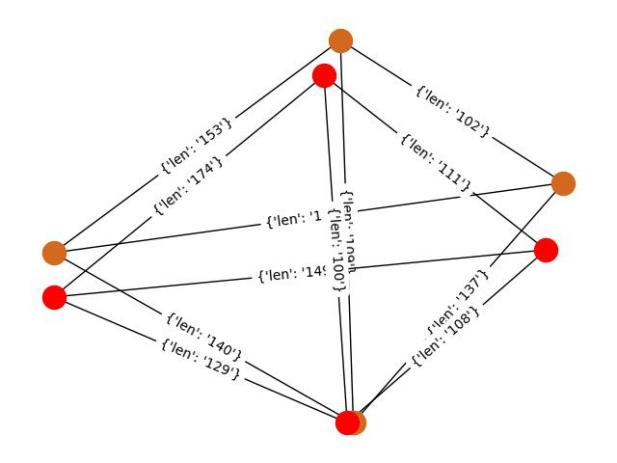}
		\caption*{(c) Intel-image-classification0}
		\label{datakm2}
	\end{minipage}
	\begin{minipage}{0.24\linewidth}
		\centering
		\includegraphics[width=0.99\linewidth]{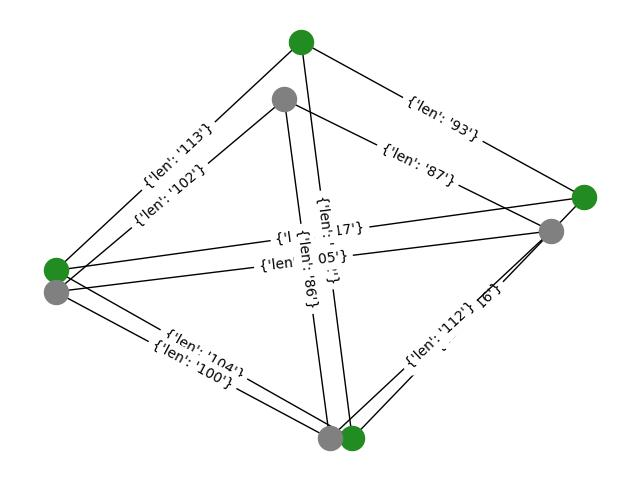}
		\caption*{(d) The-simpsons-character4}
		\label{datakm3}
	\end{minipage}

	\begin{minipage}{0.33\linewidth}
		\centering
		\includegraphics[width=0.99\linewidth]{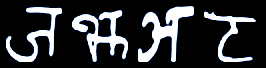}
		\caption*{(e) Devanagari-character7}
		\label{datasample2}
	\end{minipage}
	\begin{minipage}{0.33\linewidth}
		\centering
		\includegraphics[width=0.99\linewidth]{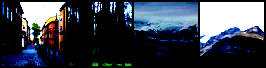}
		\caption*{(f) Intel-image-classification0}
		\label{datasample3}
	\end{minipage}
	\begin{minipage}{0.33\linewidth}
		\centering
		\includegraphics[width=0.9\linewidth]{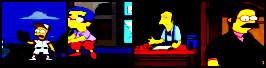}
		\caption*{(g) The-simpsons-character4}
		\label{datasample4}
	\end{minipage}
	\caption{Visualization Explanation of Three Query Datasets. Figures (e), (f), and (g) are example images of these semantically different query datasets, (a) is the distance graphs obtained by calculating the Euclidean distances between their features, (b), (c) and (d) are corresponding distance graphs obtained by sampling two similar sets of images separately.}
	\label{datakn}
\end{figure*}

There are two key query parameters may affect the retrieval performance,
and they are the number $q_n$ of query images per class, and the knowledge embedding length $q_l$.  We conduct an ablation study to illustrate the retrieval behaviors with varying combinations of them.
For each test task, we randomly select a range of 2 to 8
query images per class as the query dataset, and tested their performance on the NNR task, as illustrated in Fig. \ref{qpara}. To investigate the affection of knowledge embedding length on retrieval accuracy, we trained the retrieval framework with knowledge vectors of 128, 256, and 512 dimensions.

It can be seen from Fig.\ref{qpara} that a large value of $q_n$ ( x-axis )
leads to an increasing retrieval accuracy in most cases. For the 128, 256, and 512-dimensional query vectors, the accuracies for query
dataset consisting of 8 images increase by 11.03\%, 10.86\%,
and 8.79\% compared to those for query dataset with 2 image per category.  There is no doubt about this, because as $q_n$ increases, the retrieval framwork gathers more comprehensive knowledge.

Comparatively, the length of the query vector has a minimal impact on retrieval accuracy. Specifically, when there are 7 query images per class, the 512-dimensional embeddings yield the highest accuracy, surpassing the lowest-performing 128-dimensional results by 5.18\%.  However, with only 3 query images per class, the 128-dimensional embeddings demonstrate superior performance. The relative underperformance of 128-dimensional vectors in most case can likely be traced to their reduced capacity to capture finer details. In conclusion, the 256-dimensional knowledge vectors are sufficient to achieve alignment in the NNR task.

\subsubsection{Visualization Explanation of Query Samples}
We randomly selected three query tasks and presented partial images from the query dataset along with the visualization results of  inter-category feature distance, as depicted in Fig. \ref{datakn}. For each query task, we select a representative image from each category, as shown in Figs. \ref{datakn} (e), (f), (g), respectively. Subsequently, leveraging the aforementioned backbone,  we computed the Euclidean distances between the features for each set of images and presented these in a graphical format in Fig. \ref{datakn} (a). The blue, red, and green dots denote the visualization results of Figs .\ref{datakn} (e),  (f) and  (g), respectively. Furthermore, we sampled another set of images for each query task and showed the corresponding distance graphs  in (b), (c), and (d), respectively. The results are striking: a notable disparity in feature distances across various tasks is evident in Fig.\ref{datakn} (a), while for the same task, the feature distances among different categories exhibit a remarkable consistency between different images, as shown in Figs. \ref{datakn} (b), (c), and (d). This observation underscores the significant variability in query features of various categories across different task.
The insights gleaned from this analysis are vital for constructing knowledge vectors in query tasks, adeptly capturing the nuanced semantic interactions among features of different categories, and thus enable precise neural network retrieval. 
The ablation study presented in the main text substantiates the validity of our underlying motivation, affirming the robustness of the query knowledge extractor $Q_{EXT}$.

\begin{figure*}[!h]
	\centering
	\begin{minipage}{0.3\linewidth}
		\centering
		\includegraphics[width=1.00\linewidth]{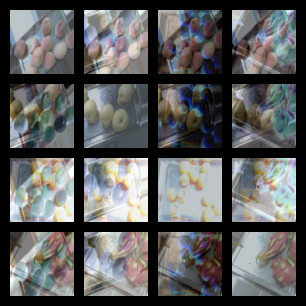}
		\caption*{(a) Fruit-recognition.}
		\label{b1}
	\end{minipage}
	\begin{minipage}{0.3\linewidth}
		\centering
		\includegraphics[width=1.00\linewidth]{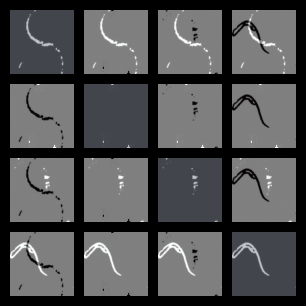}
		\caption*{(b) Numta-BengaliAI.}
		\label{b2}
	\end{minipage}
	\begin{minipage}{0.3\linewidth}
		\centering
		\includegraphics[width=1.00\linewidth]{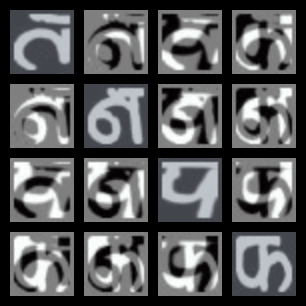}
		\caption*{(c) Devanagari-character.}
		\label{b3}
	\end{minipage}
	
	\caption{Visualization of Boundary decision samples of different datasets.}
	\label{boudec}
\end{figure*}

\subsection{Detailed Analysis of Model Knowledge Representation}

\begin{figure*}[!h]
	\centering
	
	\begin{minipage}{0.95\linewidth}
		\centering
		\includegraphics[width=0.24\linewidth]{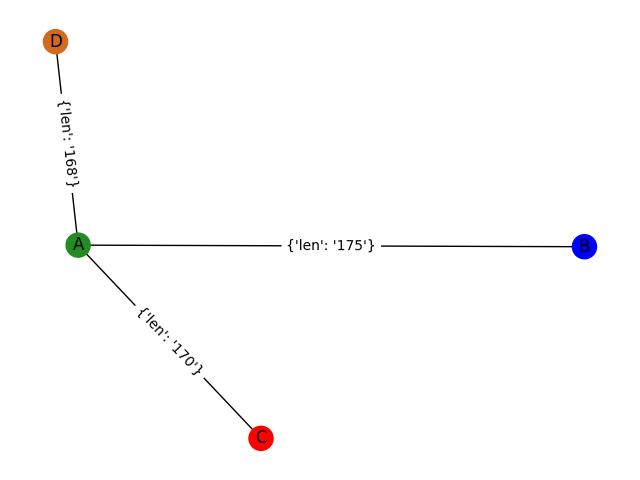}
		\centering
		\includegraphics[width=0.24\linewidth]{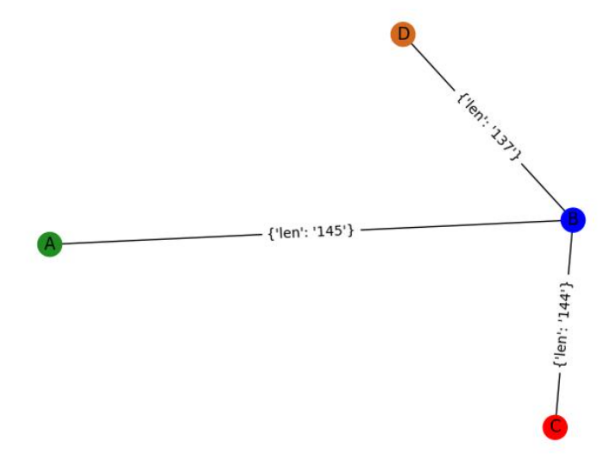}
		\centering
		\includegraphics[width=0.24\linewidth]{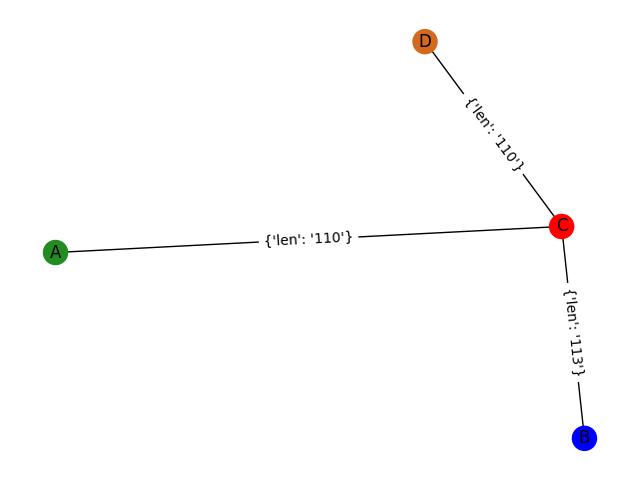}
		\centering
		\includegraphics[width=0.24\linewidth]{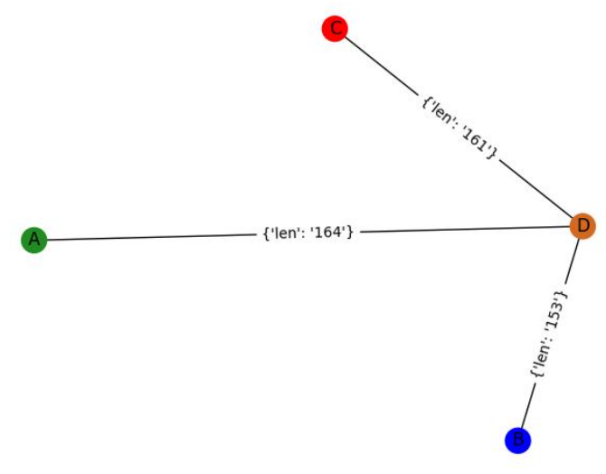}
		\caption*{(a) Simpsons-challenge Classification.}
		\label{vis1}
	\end{minipage}
	
	\begin{minipage}{0.9\linewidth}
		\centering
		\includegraphics[width=0.24\linewidth]{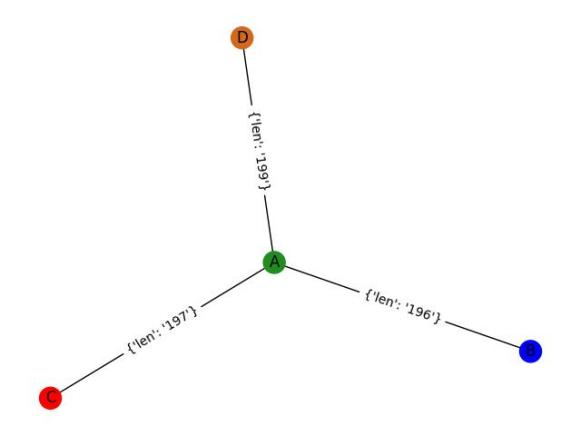}
		\centering
		\includegraphics[width=0.24\linewidth]{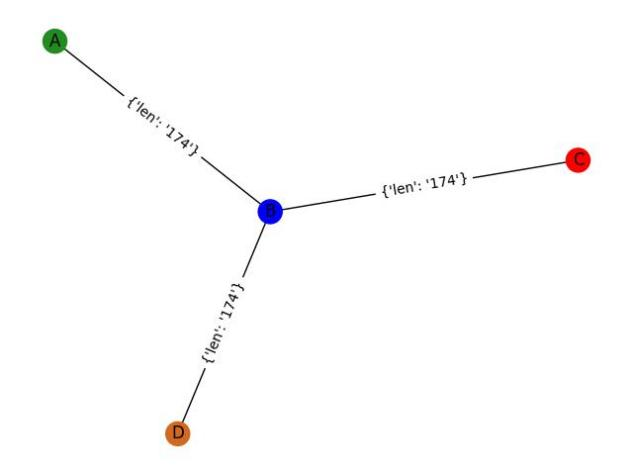}
		\centering
		\includegraphics[width=0.24\linewidth]{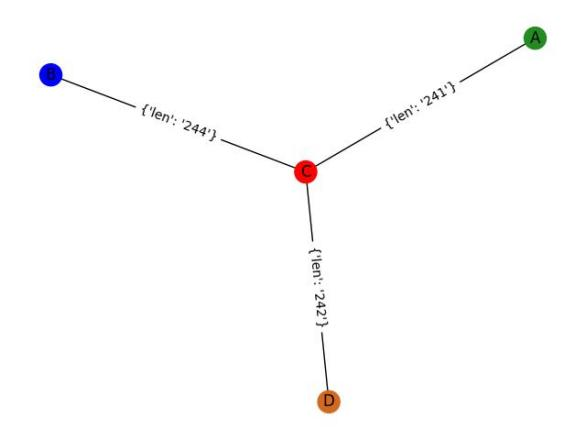}
		\centering
		\includegraphics[width=0.24\linewidth]{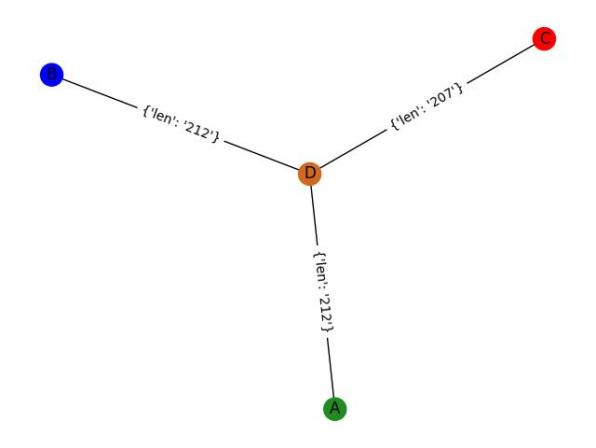}
		\caption*{(b) Devanagari-character Classification.}
		\label{vis2}
	\end{minipage}
	
	\begin{minipage}{0.9\linewidth}
		\centering
		\includegraphics[width=0.24\linewidth]{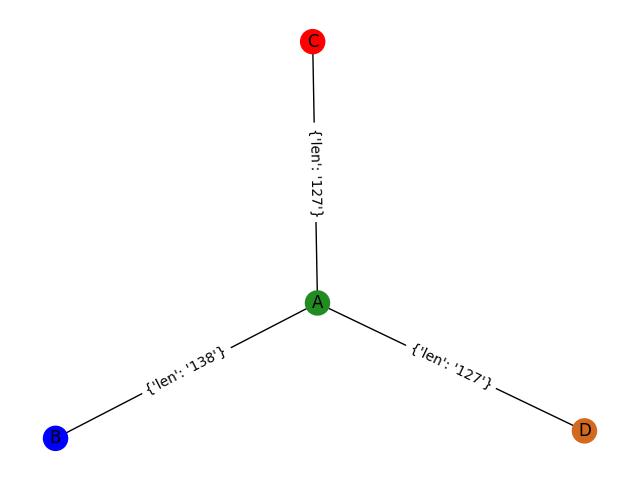}
		\centering
		\includegraphics[width=0.24\linewidth]{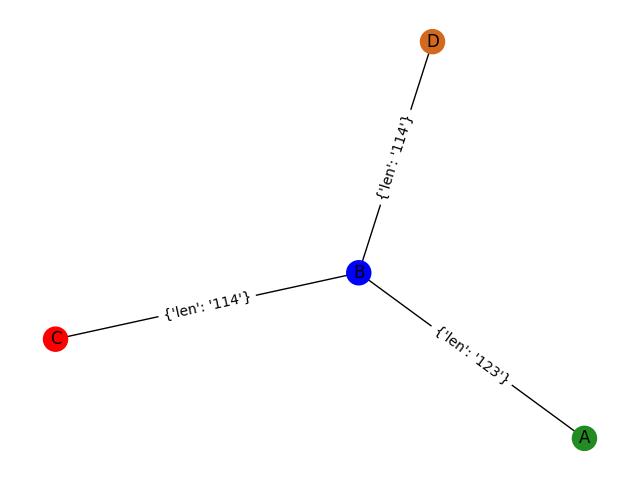}
		\centering
		\includegraphics[width=0.24\linewidth]{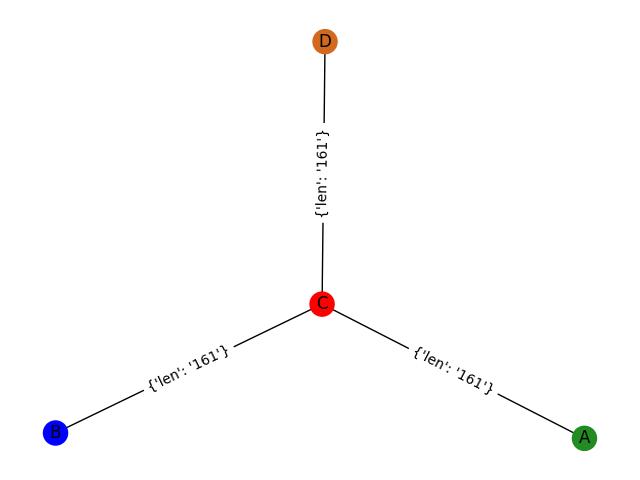}
		\includegraphics[width=0.24\linewidth]{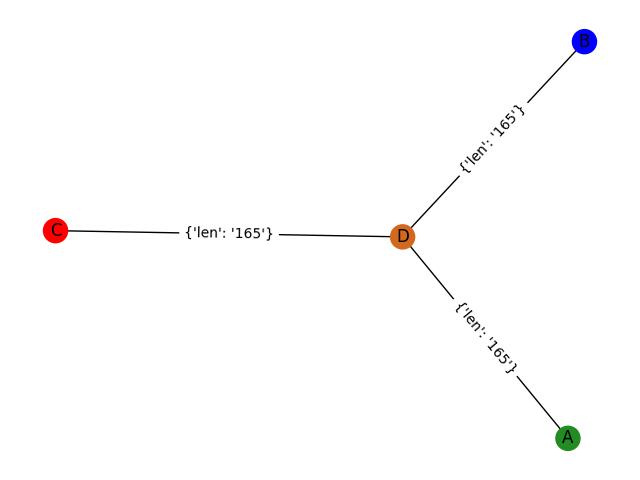}
		\caption*{(c) Intel-image Classification.}
	\end{minipage}
	\caption{Visualization of model knowledge graph sets. From top to bottom are the knowledge graph sets of three models. From left to right, there are sub-graphs of different categories. }
	\label{graphset}
\end{figure*}

\subsubsection{Visualization Explanation of Boundary Decision Samples}
We employ randomly selected training samples from the respective models as probe samples to vividly demonstrate the significance of boundary samples. Fig. \ref{boudec} offers a partial visualization, positioning the probe samples along the main diagonal, which shows the most distinctive characters of each category. Additional images depict the  decision boundary samples from the class indicated on the horizontal axis, the origin class, to the class on the vertical axis, the target class. Figs. \ref{boudec} (a), (b), and (c) showcase the visualization results for   ``Fruit-recognition" model,  ``Numta-BengaliAI classification" model and ``Devanagari-character classification" model, respectively.

In Fig. \ref{boudec} (a), sub-images in 2nd, 3rd, and 4th rows of the first column highlight the contours of peaches in the first row, signifying their distinctive features within the first category.  Similar patterns are observed in other images, emphasize that boundary samples encapsulate the quintessential traits of both the source and target categories.
The same conclusion can be drawn from Figs. \ref{boudec} (b) and (c), where every boundary sample captures the most prominent attribute of the original sample in white, while the black characters represent the most distinctive traits of the target category.

\subsubsection{Visualization Explanation of Model Knowledge Graph Set.}
We visualized the graph sets of model knowledge that defined as $G^\Phi=\{G_a,G_b,G_c,G_d\}$.
The graphs in Fig. \ref{graphset} correspond to the outcomes of Models  ``Simpsons-challenge-g03", ``Devanagari-character7", and ``Intel-image-classifi0", respectively.  For each classification model, we calculated the L2 norm of the transfer vectors $\mathbf{r}$ within the graph.   Moving from left to right, represents the connection  in $G_a,G_b,G_c,G_d$ respectively. For instance, in the image positioned at the leftmost of Fig. \ref{graphset}(a), the distance between the green segment $A$ and the red segment $C$ is 170, which measures the L2 norm of $\mathbf{r_a^c}$ being 170. 


\subsubsection{Implementation of External Probe Datasets} In our NNR experiments, the probe dataset is sourced from the Know2Vec training set. We select images randomly to feed into the target model, and the resulting diverse category samples form our probe dataset. For boundary sample generation, we require one sample from each of the source and target categories. Details of this method can be found in reference \cite{KR}.
In the SF-MTE experiment, we draw our probe dataset from several well-known datasets:
OfficeHome \cite{officehome}, CUB2011 \cite{cub2011}, VLCS \cite{vlcs}, ImageNet \cite{imagenet}. Despite the constraints imposed by the limited size of the probe dataset, which allows us to collect probe samples from only a subset of categories for each model, including boundary samples.  We believe that even with limited model knowledge, we can achieve effective model retrieval, as supported by our main text results.


\subsection{Detailed Analysis of Influence Factors in Model Representation}
We conducted ablation studies to examine the performance of Know2Vec with varying numbers of pre-trained models (PTMs). By selecting a dynamic number of models, we assessed Know2Vec's effectiveness on the Aircraft and DTD datasets, as measured by Pearson and Spearman correlation coefficients.

As shown in Table \ref{dmte}, Know2Vec consistently showed a high level of performance as the size of model zoo increased from 10 to 40. Generally, the addition of more models led to a slight decrease in Know2Vec's performance, reflecting the increased challenge in model selection. Notably, on the Aircraft dataset, Know2Vec's performance improved when the model library size exceeded form 20 models to 30 models, suggesting that the initial subset may have lacked some models on which Know2Vec excels. This variability in performance is likely attributed to the composition of our training dataset.

\begin{table}[h]
	\centering
	\begin{tabular}{c|c|c|c|c} 
		\hline
		$num$&\multicolumn{2}{c}{Aircraft}&\multicolumn{2}{c}{DTD}\\
		\hline
		&Pearson & Spearman&Pearson & Spearman\\
		\hline
		10&0.6128&0.6121 &0.6697 &0.7333\\
		20&0.5739& 0.6092 &0.6512 &0.6400\\
		30& 0.6223 & 0.6902 &0.5762 & 0.5583\\
		40&0.5426& 0.5212& 0.5691& 0.6100\\
		\hline
	\end{tabular}
	\caption{SF-MTE performance with the dynamically models num.}
	\label{dmte}
\end{table}


\section{Acknowledgments}

This work was supported in part by  National Natural Science Foundation of China under grant No. 62371450, Ningbo Natural Science Foundation under contract 2022J189,  and the Cooperation Project Between Chongqing Municipal Undergraduate Universities and Institutes affiliated to Chinese Academy of Sciences under grant HZ2021015. Additionally, this work was supported by the Chinese Academy of Sciences under grant No. XDB0690302.

\bibliography{szy}

\end{document}